\newif\ifdraft \draftfalse
\newif\iffull \fulltrue

\documentclass[11pt]{article}

\usepackage{subfigure}
\usepackage[algo2e]{algorithm2e}
\usepackage{fullpage}
\usepackage{amsmath, amssymb, amsthm}
\usepackage{bbm}
\usepackage{mathtools}
\usepackage{color}
\usepackage{kpfonts}
\usepackage{xcolor}

\usepackage{xspace}
\usepackage{hyperref}
\usepackage{cleveref}
\usepackage{thmtools, thm-restate}
\usepackage{natbib}

\usepackage{bm}
\usepackage{bbm}
\usepackage{tikz}
\usepackage{mathrsfs}
\usepackage{commath}
\usepackage{mathtools}
\usepackage{color}
\usepackage{xcolor}

\usepackage{multicol}
\usepackage{xpatch}
\usepackage{thmtools}
\usepackage{apptools}

\makeatletter
\xpatchcmd{\thmt@restatable}% Edit \thmt@restatable
{\csname #2\@xa\endcsname\ifx\@nx#1\@nx\else[{#1}]\fi}% Replace this code
{\IfAppendix{\csname #2\@xa\endcsname}{\csname #2\@xa\endcsname\ifx\@nx#1\@nx\else[{#1}]\fi}}% with this code
{}{} % execute code for success/failure instances
\makeatother

\definecolor{DarkGreen}{rgb}{0.1,0.5,0.1}
\definecolor{DarkRed}{rgb}{0.5,0.1,0.1}
\definecolor{DarkBlue}{rgb}{0.1,0.1,0.5}
\usepackage[]{hyperref}
\hypersetup{
    unicode=false,          % non-Latin characters in Acrobat's bookmarks
    pdftoolbar=true,        % show Acrobat toolbar?
    pdfmenubar=true,        % show Acrobat menu?
    pdffitwindow=false,      % page fit to window when opened
    pdftitle={},    % title
    pdfauthor={}
    pdfsubject={},   % subject of the document
    pdfnewwindow=true,      % links in new window
    pdfkeywords={keywords}, % list of keywords
    colorlinks=true,       % false: boxed links; true: colored links
    linkcolor=DarkRed,          % color of internal links
    citecolor=DarkGreen,        % color of links to bibliography
    filecolor=DarkRed,      % color of file links
    urlcolor=DarkBlue,          % color of external links
}
\usepackage{xspace}
\usepackage{cleveref}
\usepackage{thmtools, thm-restate}
\usepackage{natbib}

\usepackage{hyperref}

\newcommand{\cj}[1]{\ifdraft \textcolor{purple}{[Chris: #1]}\fi}
\newcommand{\ar}[1]{\ifdraft \textcolor{red}{[Aaron: #1]}\fi}
\newcommand{\mk}[1]{\ifdraft \textcolor{orange}{[Michael: #1]}\fi}

\newcommand{\algo}{\bm{L}}
\newcommand{\E}{\mathop{\mathbb{E}}}
\newcommand{\cO}{\bm{O}}
\newcommand{\DE}{\mathbf{DistanceEstimator}}

\newtheorem{definition}{Definition}
\newtheorem{corollary}{Corollary}
\newtheorem{lemma}{Lemma}
\newtheorem{theorem}{Theorem}
\newtheorem{remark}{Remark}

\title{Online Learning with an Unknown Fairness Metric}

\author{Stephen Gillen\thanks{Department of Mathematics, University of Pennsylvania.} \and Christopher Jung\thanks{Department of Computer and Information Sciences, University of Pennsylvania. Supported in part by a grant from the Quattrone Center for the Fair Administration of Justice.} \and Michael Kearns\thanks{Department of Computer and Information Sciences, University of Pennsylvania.} \and Aaron Roth\thanks{Department of Computer and Information Sciences, University of Pennsylvania. Supported in part by grants from the DARPA Brandeis project, the Sloan Foundation, and NSF grants CNS-1513694 and CNS-1253345.}}

\begin{document}

\maketitle

\begin{abstract}
We consider the problem of online learning in the linear contextual bandits setting,
but in which there are also strong {\em individual fairness\/} constraints governed by an unknown
similarity metric. These constraints demand that we select similar actions or individuals
with approximately equal probability \citep{DHPRZ12}, which may be at odds with optimizing reward,
thus modeling settings where profit and social policy are in tension.
We assume we learn about an unknown
Mahalanobis
similarity metric
from only weak feedback that identifies fairness violations, but does not quantify their extent. This is intended to represent the interventions of a regulator who ``knows unfairness when he sees it'' but nevertheless cannot enunciate a quantitative fairness metric over individuals.
Our main result is an algorithm
in the adversarial context setting that
has a number of
fairness violations that depends only logarithmically on $T$,
while obtaining an optimal $O(\sqrt{T})$ regret bound to the best fair policy.
\end{abstract}

\thispagestyle{empty} \setcounter{page}{0}
\clearpage

\section{Introduction}

The last several years have seen an explosion of work studying the problem of fairness in machine learning. Yet there
remains little agreement about what ``fairness'' should mean in different contexts.
In broad strokes, the literature can be divided into two families of fairness definitions: those aiming at \emph{group} fairness,
and those aiming at \emph{individual} fairness.

Group fairness definitions are aggegrate in nature:
they partition individuals into some collection of \emph{protected groups} (say by race or gender), specify some
statistic of interest (say, positive classification rate or false positive rate), and then require that a learning algorithm equalize this quantity across the protected groups.
On the other hand, individual fairness definitions ask for
some constraint that binds on the individual level, rather than only over averages of people.
Often, these constraints have the semantics that ``similar people should be treated similarly'' \cite{DHPRZ12}.

Individual fairness definitions have substantially stronger semantics and demands
than group definitions of fairness. For example, \cite{DHPRZ12} lay out a compendium of ways in which group fairness definitions are unsatisfying. %Perhaps the most salient is that because group fairness constraints bind only over population averages, they do not prevent substantial discrimination against structured subgroups in the population, so long as this discrimination is ``cancelled out'' by favoritism to some other subgroup. This phenomenon is dubbed ``fairness gerrymandering'' by \cite{KNRW17}, who demonstrate that it arises naturally when applying standard machine learning algorithms.
Yet despite these weaknesses, group fairness definitions are by far the most prevalent in the
literature (see e.g. \cite{kamiran2012data,hajian2013methodology,KMR16,HPS,FSV16,zafar2017fairness,Chou17} and \cite{berksurvey} for a survey).
This is in large part because notions of individual fairness require making stronger assumptions
on the setting under consideration. In particular, the definition from \cite{DHPRZ12}
requires that the algorithm designer know a ``task-specific fairness metric.''

Learning problems over individuals are also often implicitly accompanied
by some notion of {\em merit\/}, embedded in the objective function of
the learning problem. For example, in a lending setting we might posit that each loan applicant is either
``creditworthy'' and will repay a loan, or is not creditworthy and will default --- which is what we are trying to predict. \cite{JosephKMR16} take the approach that this measure of merit --- already present in the model, although initially unknown to the learner --- can be taken to be the similarity metric in the definition of \cite{DHPRZ12}, requiring informally that creditworthy individuals have at least the same probability of being accepted for loans as
defaulting individuals. (The implicit and coarse fairness metric here assigns distance zero between
pairs of creditworthy individuals and pairs of defaulting individuals, and
some non-zero distance between a creditworthy and a defaulting individual.) This resolves the problem of how one should discover the ``fairness metric'', but results in a notion of fairness that is necessarily aligned with
the notion of ``merit'' (creditworthiness) that we are trying to predict.
%More generally,
%the definition from \cite{DHPRZ12} requires that the algorithm designer know some task-specific
%fairness metric defined over (pairs of) individuals.

However, there are many settings in which the notion of merit we wish to predict may be different or
even at odds with the notion of fairness we would like to enforce. For example,
notions of fairness aimed at rectifying societal inequities that result from historical
discrimination can aim to favor the disadvantaged population (say, in college admissions),
even if the performance of the admitted members of that population can be expected to be lower than that of the advantaged population.
Similarly, we might desire a fairness metric incorporating only those attributes that
individuals can change in principle (and thus excluding ones like race, age and gender),
and that further expresses what are and are not meaningful differences between individuals,
outside the context of any particular prediction problem.
These kinds of fairness desiderata can still be expressed as an instantiation of
the definition from \cite{DHPRZ12}, but with a task-specific fairness metric separate from the
notion of merit we are trying to predict.

%For example, in college
%admissions it could be that applicants from financially advantaged families have inflated
%SAT scores --- not due to some underlying merit, but simply because they can afford
%test preparation courses and multiple exam retakes that less advantaged applicants cannot.
%In this case the ``merit'' we wish to predict might be some notion of collegiate success, but
%the fairness metric should discount the SAT scores of wealthy applicants. But we might not
%realize this initially, and have to learn it over time from observed applications.
%Similarly, the desired fairness metric might express some form of affirmative
%action, such as deliberately investing in small businesses in poorer communities, that
%might be at odds with the predictive or merit goal (such as return on investment).
%Even if this fairness metric is expressed as written policy, it actual application to particular
%instances might have to be learned over time (e.g. through policy or legal decisions).

In this paper, we revisit the individual fairness definition from \cite{DHPRZ12}.
This definition requires that pairs of individuals who are close in the fairness metric must be
treated ``similarly'' (e.g. in an allocation problem such as lending, served with similar probability).
We investigate the extent to which it is possible to satisfy this
fairness constraint while simultaneously solving an online learning problem,
when the underlying fairness metric is Mahalanobis but \emph{not} known to the
learning algorithm, and may also be
in tension with the learning problem. One conceptual problem with metric-based definitions, that we seek to address, is that it may be difficult for anyone to actually precisely express a quantitative metric over individuals --- but they nevertheless might ``know unfairness when they see it.''
We therefore assume that the algorithm has access to an oracle
% (say, a friend in the philosophy department)
that knows intuitively what it means to be fair, but cannot explicitly enunciate
the fairness metric.
Instead, given observed actions,
the oracle can specify whether they were fair or not, % with respect to the underlying fairness metric.
% Thus the learner has access to ``fairness queries'' \ar{Hm -- if we call them queries, it makes it sound like the learner can ask whether distributions are fair, without acting on them. But in our model, we just get feedback for what we have done. Call them something else?},
and the goal is to obtain
low regret in the online learning problem --- measured with respect to the best \emph{fair} policy ---
while also limiting violations of individual fairness during the learning process.

\subsection{Our Results and Techniques}

We study the standard linear contextual bandit setting. In rounds $t = 1, \ldots, T$, a learner observes
arbitrary and possibly adversarially selected $d$-dimensional contexts, each corresponding to one of $k$ actions. The reward for each action is (in expectation) an unknown linear function of the contexts. The learner seeks to minimize its regret.

The learner also wishes to satisfy \emph{fairness constraints}, defined with respect to an unknown distance function defined over contexts. The constraint requires that the difference between the probabilities that any two actions are taken is bounded by the distance between their contexts. The learner has no initial knowledge of the distance function. Instead, after the learner makes its decisions according to some probability distribution $\pi^t$ at round $t$,
it receives feedback specifying for which pairs of contexts the fairness constraint was violated. Our goal in designing a learner is to simultaneously guarantee near-optimal regret in the contextual bandit problem (with respect to the best \emph{fair} policy), while violating the fairness constraints as infrequently as possible. Our main result is a computationally efficient algorithm that guarantees this for a large class of distance functions known as \emph{Mahalanobis distances} (these can be expressed as $d(x_1,x_2) = ||Ax_1 - Ax_2||_2$ for some matrix $A$).

\noindent
\textbf{Theorem} (Informal): There is a computationally efficient learning algorithm $\algo$ in our setting that guarantees that for any Mahalanobis distance, any time horizon $T$, and any error tolerance $\epsilon$:
\begin{enumerate}
\item (Learning) With high probability, $\algo$ obtains regret  $\tilde O\left(k^2d^2 \log\left(T \right) + d\sqrt{T}\right)$ to the best fair policy (See Theorem \ref{thm:lastregret} for a precise statement.)
\item (Fairness) With probability $1$, $\algo$ violates the unknown fairness constraints
	by more than $\epsilon$ on at most $O \left(k^2 d^2 \log(d/\epsilon)\right)$ many rounds. (Theorem \ref{thm:fairfull}.)
\end{enumerate}

We note that the quoted regret bound requires setting $\epsilon = O(1/T)$, and so this implies a number of fairness violations of magnitude more than $1/T$ that is bounded by a function growing logarithmically in $T$. Other tradeoffs between regret and fairness violations are possible.

These two goals: of obtaining low regret, and violating the unknown constraint a small number of times --- are seemingly in tension. A standard technique for obtaining a mistake bound with respect to fairness violations would be to play a ``halving algorithm'', which would always act as if the unknown metric is at the center of the current version space (the set of metrics consistent with the feedback observed thus far) --- so that mistakes necessarily remove a non-trivial fraction of the version space, making progress. On the other hand, a standard technique for obtaining a diminishing regret bound is to play ``optimistically'' -- i.e. to act as if the unknown metric is the point in the version space that would allow for the largest possible reward. But ``optimistic'' points are necessarily at the boundary of the version space, and when they are falsified, the corresponding mistakes do not necessarily reduce the version space by a constant fraction.

We prove our theorem in two steps. First, in Section \ref{sec:knownobj}, we consider the simpler problem in which the linear objective of the contextual bandit problem is known, and the distance function is all that is unknown. In this simpler case, we show how to obtain a bound on the number of fairness violations using a linear-programming based reduction to a recent algorithm which has a mistake bound for learning a linear function with a particularly weak form of feedback \cite{LobelLV17}. A complication is that our algorithm does not receive all of the feedback that the algorithm of \cite{LobelLV17} expects. We need to use the structure of our linear program to argue that this is ok. Then, in Section \ref{sec:full}, we give our algorithm for the complete problem, using large portions of the machinery we develop in Section \ref{sec:knownobj}.

We note that in a non-adversarial setting, in which contexts are drawn from a distribution,
the algorithm of \cite{LobelLV17} could be more simply applied along with standard
techniques for contextual bandit learning to give an explore-then-exploit style algorithm.
This algorithm
would obtain bounded (but suboptimal) regret, and a number of fairness violations that grows as a root of $T$.
The principal advantages of our approach are that we are able to give a number
of fairness violations that has only \emph{logarithmic} dependence on $T$,
while tolerating contexts that are chosen adversarially, all while obtaining an optimal $O(\sqrt{T})$ regret bound to the best fair policy.

\subsection{Additional Related Work}

There are two papers, written concurrently to ours, that tackle orthogonal issues in metric-fair learning. \cite{RY18} consider the problem of \emph{generalization} when performing learning subject to a known metric constraint. They show that it is possible to prove relaxed PAC-style generalization bounds without any assumptions on the metric, and that for worst-case metrics, learning subject to a metric constraint can be computationally hard, even when the unconstrained learning problem is easy. In contrast, our work focuses on online learning with an \emph{unknown} metric constraint. Our results imply similar generalization properties via standard online-to-offline reductions, but only for the class of metrics we study. \cite{KRR18} considers a group-fairness like relaxation of metric-fairness, asking that on average, individuals in pre-specified groups are classified with probabilities proportional to the average distance between individuals in those groups. They show how to learn such classifiers in the offline setting, given access to an oracle which can evaluate the distance between two individuals according to the metric (allowing for unbiased noise). The similarity to our work is that we also consider access to the fairness metric via an oracle, but our oracle is substantially weaker, and does not provide numeric valued output. 

There are also several papers in the algorithmic fairness literature that are thematically related to ours, in that they both aim to bridge the gap between group notions of fairness (which can be semantically unsatisfying) and individual notions of fairness (which require very strong assumptions). \cite{ZWSPD13} attempt
to automatically learn a representation for the data in a batch learning
problem (and hence, implicitly, a similarity metric) that causes a classifier to label an equal proportion of two protected groups as positive.
They provide a heuristic approach and an experimental evaluation.
% \ar{If someone has time to glance at this paper, please verify this is what they do.}
Two recent papers (\cite{KNRW17} and \cite{HKRR17})
take the approach of asking for a group notion of fairness,
but over exponentially many implicitly defined protected groups, thus
mitigating what \cite{KNRW17} call the ``fairness gerrymandering'' problem,
which is one of the principal weaknesses of group
fairness definitions. Both papers give polynomial time reductions which yield efficient algorithms whenever a corresponding agnostic learning problem is solvable. In contrast, in this paper, we
take a different approach: we attempt to directly
satisfy the original definition of individual
fairness from \cite{DHPRZ12}, but with substantially less information about the underlying similarity metric.

Starting with \cite{JosephKMR16}, several papers have studied notions of fairness in classic and contextual bandit problems. \cite{JosephKMR16} study a notion of ``meritocratic'' fairness in the contextual bandit setting, and prove upper and lower bounds on the regret achievable by algorithms that must be ``fair'' at every round. This can be viewed as a variant of the \cite{DHPRZ12} notion of fairness, in which the expected reward of each action is used to define the ``fairness metric''. The algorithm does not originally know this metric, but must discover it through experimentation. \cite{infiniterawls} extend the work of \cite{JosephKMR16} to the setting in which the algorithm is faced with a continuum of options at each time step, and give improved bounds for the \emph{linear} contextual bandit case. \cite{JJKMR17} extend this line of work to the reinforcement learning setting in which the actions of the algorithm can impact its environment. Finally, \cite{LRDP17} consider a notion of fairness based on calibration in the simple stochastic bandit setting.

There is a large literature that focuses on
learning Mahalanobis distances --- see \cite{Kul13} for a survey.
In this literature, the closest paper to our work focuses on \emph{online}
learning of Mahalanobis distances (\cite{JKDG09}). However, this result is in a very different
setting from the one we consider here. In \cite{JKDG09}, the algorithm is repeatedly
given pairs of points, and needs to predict their distance.
It then learns their true distance, and aims to minimize its squared loss.
In contrast, in our paper, the main objective of the learning algorithm is orthogonal to
the metric learning problem --- i.e. to minimize regret in the linear contextual bandit problem, but while
simultaneously learning and obeying a fairness constraint, and only from weak feedback noting violations of fairness.

\section{Model and Preliminaries}
\label{sec:prelims}
\subsection{Linear Contextual Bandits}
We study algorithms that operate in the \emph{linear contextual bandits} setting. A linear contextual bandit problem is parameterized by an unknown vector of linear coefficients $\theta \in \mathbb{R}^d$, with $||\theta||_2 \leq 1$. Algorithms in this setting operate in \emph{rounds} $t = 1, \ldots, T$. In each round $t$, an algorithm $\algo$ observes $k$ \emph{contexts} $x^t_1,\ldots,x^t_k \in \mathbb{R}^d$, scaled such that $||x^t_i||_2 \leq 1$. We write $x^t = (x^t_1,\ldots,x^t_k)$ to denote the entire set of contexts observed at round $t$. After observing the contexts, the algorithm chooses an action $i^t$. \ar{changing notation for selected action to something more standard --- need to propagate.}After choosing an action, the algorithm obtains some stochastic \emph{reward} $r^t_{i^t}$ such that $r^t_{i^t}$ is subgaussian\footnote{A random variable $X$ with $\mu = \mathbb{E}[X]$ is sub-gaussian, if for all $t \in \mathbb{R}$, $\mathbb{E}[e^{t(X-\mu)}] \le e^{\frac{t^2}{2}}$.} and $\E[r^t_{i^t}] = \langle x^t_{i^t}, \theta \rangle$. \ar{We need to be consistent about how we denote inner products, which varies throughout the document. I recommend this notation. Check out the tex --- it uses commands langle, rangle, not just $\leq$ and $\geq$ signs.} The algorithm does not observe the reward for the actions not chosen. When the action $i^t$ is clear from context, and write $r^t$ instead of $r^t_{i^t}$.

\begin{remark}
For simplicity, we consider algorithms that select only a \emph{single} action at every round. However, this assumption is not necessary. In the appendix, we show how our results extend to the case in which the algorithm can choose any number of actions at each round. This assumption is sometimes more natural: for example, in a lending scenario, a bank may wish to make loans to as many individuals as will be profitable, without a budget constraint.
\end{remark}

\mk{See later comments on making history dependence implicit to reduce clutter}
In this paper, we will be discussing algorithms $\algo$ that are necessarily randomized. To formalize this, we denote a history including everything observed by the algorithm up through but not including round $t$ as $h^t = ((x^1,i^1,r^1),\ldots,(x^{t-1},i^{t-1},r^{t-1}))$ The space of such histories is denoted by $\mathcal{H}^{t} = (\mathbb{R}^{d\times k} \times [k] \times \mathbb{R})^{t-1}$. An algorithm $\algo$ is defined by a sequence of functions $f^1,\ldots,f^T$ each mapping histories and observed contexts to probability distributions over actions:
$$f^t:\mathcal{H}^t \times \mathbb{R}^{d \times k} \rightarrow \Delta [k].$$ We write $\pi^t$ to denote the probability distribution over actions that $\algo$ plays at round $t$: $\pi^t = f^t(h^t, x^t)$. We view $\pi^t$ as a vector over $[0,1]^k$, and so $\pi^t_i$ denotes the probability that $\algo$ plays action $i$ at round $t$. We denote the expected reward of the algorithm at day $t$ as $\E[r^t] = \E_{i \sim \pi^t}[r_i^t]$. It will sometimes also be useful to refer to the vector of expected rewards across all actions on day $t$. We denote it as
$$\bar{r}^t = (\langle x_1^t, \theta \rangle, \ldots, \langle x_k^t, \theta \rangle).$$
Note that this vector is of course unknown to the algorithm.
\subsection{Fairness Constraints and Feedback}
We study algorithms that are constrained to behave \emph{fairly} in some manner. We adapt the definition of fairness from \cite{DHPRZ12} that asserts, informally, that ``similar individuals should be treated similarly''. We imagine that the decisions that our contextual bandit algorithm $\algo$ makes correspond to individuals, and that the contexts $x_i^t$ correspond to features pertaining to individuals. We adopt the following (specialization of) the fairness definition from Dwork et al, which is parameterized by a distance function $d:\mathbb{R}^d\times\mathbb{R}^d\rightarrow \mathbb{R}$.

\begin{definition}[\cite{DHPRZ12}]
Algorithm $\algo$ is Lipschitz-fair on round $t$ with respect to distance function $d$ if for all pairs of individuals $i, j$:
$$|\pi^t_i - \pi^t_j| \leq d(x_i^t,x_j^t).$$
For brevity, we will often just say that the algorithm is \emph{fair} at round $t$, with the understanding that we are always talking about this one particular kind of fairness.
\end{definition}

\begin{remark}
Note that this definition requires a fairness constraint that binds between individuals at a single round $t$, but not between rounds $t$. This is for several reasons. First, at a philosophical level, we want our algorithms to be able to improve with time, without being bound by choices they made long ago before they had any information about the fairness metric. At a (related) technical level, it is easy to construct lower bound instances that certify that it is impossible to simultaneously guarantee that an algorithm has diminishing regret to the best fair policy, while violating fairness constraints (now defined as binding across rounds) a sublinear number of times.
\end{remark}

One of the main difficulties in working with Lipschitz fairness (as discussed in \cite{DHPRZ12}) is that the distance function $d$ plays a central role, but it is not clear how it should be specified. In this paper, we concern ourselves with learning $d$ from feedback. In particular, algorithms $\algo$ will have access to a \emph{fairness oracle}.

 Informally, the fairness oracle will take as input: 1) the set of choices available to $\algo$ at each round $t$, and 2) the probability distribution $\pi^t$ that $\algo$ uses to make its choices at round $t$, and returns the set of all pairs of individuals for which $\algo$ violates the fairness constraint.

\begin{definition}[Fairness Oracle]
Given a distance function $d$, a fairness oracle $\cO_d$ \ar{Note changed notation so that oracle is subscripted with $d$} is a function $\cO_d:\mathbb{R}^{d\times k}\times \Delta [k] \rightarrow 2^{[k]\times [k]}$ defined such that:
$$\cO_d(x^t, \pi^t) = \{(i,j) : |\pi^t_i - \pi^t_j| > d(x_i^t,x_j^t)\}$$
\end{definition}

Formally, algorithms $\algo$ in our setting will operate in the following environment:
\begin{definition}
\begin{enumerate}
\item An adversary fixes a linear reward function $\theta \in \mathbb{R}^d$ with $||\theta|| \leq 1$ and a distance function $d$. $\algo$ is given access to the fairness oracle $\cO_d$.
\item In rounds $t = 1$ to $T$:
\begin{enumerate}
\item The adversary chooses contexts $x^t \in \mathbb{R}^{d \times k}$ with $||x^t_i|| \leq 1$ and gives them to $\algo$.
\item $\algo$ chooses a probability distribution $\pi^t$ over actions, and chooses action $i^t \sim \pi^t$.
\item $\algo$ receives reward $r^t_{i^t}$ and observes feedback $\cO_d(\pi^t)$ from the fairness oracle.
\end{enumerate}
\end{enumerate}
\end{definition}

Because of the power of the adversary in this setting, we cannot expect algorithms that can avoid arbitrarily small violations of the fairness constraint. Instead, we will aim to limit \emph{significant} violations.

\begin{definition}
Algorithm $\algo$ is $\epsilon$-unfair on pair $(i,j)$ at round $t$ with respect to distance function $d$ if
$$|\pi^t_i-\pi^t_j| > d(x_i^t,x_j^t) + \epsilon.$$
Given a sequence of contexts and a history $h^t$ (which fixes the distribution on actions at day $t$) We write $$\mathbf{Unfair}(\algo, \epsilon, h^t) = \sum_{i=1}^{k-1} \sum_{j=i+1}^k \mathbbm{1}( |\pi^t_i-\pi^t_j| > d(x_i^t,x_j^t) + \epsilon)$$ to denote the number of pairs on which $\algo$ is $\epsilon$-unfair at round $t$.
\end{definition}

Given a distance function $d$ and a history $h^{T+1}$, the $\epsilon$-\emph{fairness loss} of an algorithm $\algo$ is the total number of pairs on which it is $\epsilon$-unfair:
$$\mathbf{FairnessLoss}(\algo,h^{T+1}, \epsilon) = \sum_{t=1}^T \mathbf{Unfair}(\algo,\epsilon, h^t)$$
For a shorthand, we'll write $\mathbf{FairnessLoss}(\algo, T, \epsilon)$.

We will aim to design algorithms $\algo$ that guarantee that their fairness loss is bounded with probability $1$ \ar{Check -- is this right?} \cj{Yes, in terms of probability. But in each round, we can do better than just limiting the number of rounds where we are off. We can bound the total number of fairness violations. I'll update this} in the worst case over the instance: i.e. in the worst case over both $\theta$ and $x^1,\ldots,x^T$, and in the worst case over the distance function $d$ (within some allowable class of distance functions -- see Section \ref{sec:mah}). \ar{Note -- new notation for fairness loss. Propagate.}

\subsection{Regret to the Best Fair Policy}
In addition to minimizing fairness loss, we wish to design algorithms that exhibit diminishing \emph{regret} to the best \emph{fair} policy. We first define a linear program that we will make use of throughout the paper. Given a vector $a \in \mathbb{R}^d$ and a vector $c \in \mathbb{R}^{k^2}$, we denote by $LP(a, c)$ the following linear program:

\begin{equation*}
\begin{aligned}
& \underset{\pi=\{p_1,\ldots,p_k\} }{\text{maximize}}
& & \sum_{i=1}^k p_i a_i \\
& \text{subject to}
& & \vert p_i - p_j \vert \le c_{i,j}, \forall (i,j)\\
&  &&\sum_{i=1}^k p_i \le 1
\end{aligned}
\end{equation*}

We write $\pi(a,c) \in \Delta [k]$ to denote an optimal solution to $LP(a,c)$. Given a set of contexts $x^t$, recall that $\bar{r}^t$ is the vector representing the expected reward corresponding to each context (according to the true, unknown linear reward function $\theta$). Similarly, we write $\bar{d}^t$ to denote the vector representing the set of distances between each pair of contexts $i,j$ (according to the true, unknown distance function $d$): $\bar{d}_{i,j}^t = d(x_i^t,x_j^t)$.

Observe that $\pi(\bar{r}^t,\bar{d}^t)$ corresponds to the distribution over actions that maximizes expected reward at round $t$, subject to satisfying the fairness constraints --- i.e. the distribution that an optimal player, with advance knowledge of $\theta$ would play, if he were not allowed to violate the fairness constraints at all. This is the benchmark with respect to which we define regret:

\begin{definition}
Given an algorithm $\algo$ ($f_1, \ldots, f_T$), a distance function $d$, a linear parameter vector $\theta$, and a history $h^{T+1}$ (which includes a set of contexts $x^1,\ldots,x^T$), its regret is defined to be:
$$\mathbf{Regret}(\algo,\theta,d,h^{T+1}) = \sum_{t=1}^T \E_{i \sim \pi(\bar{r}^t,\bar{d}^t)} [\bar{r}^t_i] - \sum_{t=1}^T \E_{i \sim f^t(h^t, x^t)} [\bar{r}^t_i]$$
\end{definition}
For shorthand, we'll write $\mathbf{Regret}(\algo, T)$. \cj{Instead of changing the notation below, introduced a short-hand notation here}

Our goal will be to design algorithms for which we can bound regret with high probability over the randomness of $h^{T+1}$ \footnote{We assume that $h^{T+1}$ is generated by algorithm $A$, meaning randomness only comes from the stochastic reward and the way in which each arm is selected according to the probability distribution calculated by the algorithm. We don't assume any distributional assumption over $x^1, \hdots, x^T$} in the worst case over $\theta$, $d$, and ($x^1, \hdots, x^T$).
\subsection{Mahalanobis Distance}
\label{sec:mah}
In this paper, we will restrict our attention to a special family of distance functions which are parameterized by a matrix $A$:
\begin{definition}[Mahalanobis distances]
A function $d:\mathbb{R}^d\times \mathbb{R}^d \rightarrow \mathbb{R}$ is a Mahalanobis distance function if there exists a matrix $A$ such that for all $x_1,x_2 \in \mathbb{R}^d$:
$$d(x_1,x_2) = ||Ax_1 - Ax_2||_2$$
where $||\cdot ||_2$ denotes Euclidean distance. Note that if $A$ is not full rank, then this does not define a metric --- but we will allow this case (and be able to handle it in our algorithmic results).
\end{definition}

Mahalanobis distances will be convenient for us to work with, because \emph{squared} Mahalanobis distances can be expressed as follows:
\begin{eqnarray*}
d(x_1,x_2)^2 &=& ||Ax_1 - Ax_2||_2^2 \\
&=& \langle A(x_1 - x_2), A(x_1 - x_2) \rangle \\
&=& (x_1-x_2)^{\top} A^{\top} A (x_1-x_2) \\
&=& \sum_{i,j=1}^d G_{i,j}(x_1-x_2)_i(x_1-x_2)_j
\end{eqnarray*}
where $G = A^{\top}A$. Observe that when $x_1$ and $x_2$ are fixed, this is a linear function in the entries of the matrix $G$. We will use this property to reason about \emph{learning} $G$, and thereby learning $d$.

\section{Warmup: The Known Objective Case}
\label{sec:knownobj}
In this section, we consider an easier case of the problem in which the linear objective function $\theta$ is known to the algorithm, and the distance function $d$ is all that is unknown. In this case, we show via a reduction to an online learning algorithm of \cite{LobelLV17}, how to simultaneously obtain a logarithmic regret bound and a logarithmic (in $T$) number of fairness violations. The analysis we do here will be useful when we solve the full version of our problem (in which $\theta$ is unknown) in Section \ref{sec:full}.

\subsection{Outline of the Solution}
\label{sec:known-overview}
Recall that since we know $\theta$, at every round $t$ after seeing the contexts, we know the vector of expected rewards $\bar{r}^t$ that we would obtain for selecting each action. Our algorithm will play at each round $t$ the distribution $\pi(\bar{r}^t, \hat{d}^t)$ that results from solving the linear program $LP(\bar{r}^t,\hat{d}^t)$, where $\hat{d}^t$ is a ``guess'' for the pairwise distances between each context $\bar{d}^t$. (Recall  that the optimal distribution to play at each round is $\pi(\bar{r}^t,\bar{d}^t)$.)

The main engine of our reduction is an efficient online learning algorithm for linear functions recently given by \cite{LobelLV17} which is further described in Section \ref{sec:de}. Their algorithm, which we refer to as $\DE$,\ar{Note I introduced a macro for the algorithm name} works in the following setting. There is an unknown vector of linear parameters $\alpha \in \mathbb{R}^m$. \ar{Using $m$ for the dimension of the learning problem in our subroutine, rather than $d$, since $m \neq d$ for us, and right now we have a notation collision. Propagate.}\cj{using $\alpha$ since $\phi$ is already use as a linear parameter for the bandit setting. Propagate. }In rounds $t$, the algorithm observes a vector of features $u^t \in \mathbb{R}^m$, and produces a prediction $g^t \in \mathbb{R}$ for the value $\langle \alpha, u^t \rangle$. After it makes its prediction, the algorithm learns whether its guess was \emph{too large} or not, but does not learn anything else about the value of $\langle \alpha, u^t \rangle$. The guarantee of the algorithm is that the number of rounds in which its prediction is off by more than $\epsilon$ is bounded by $O(m \log(m/\epsilon))$\footnote{If the algorithm also learned whether or not its guess was in error by more than $\epsilon$ at each round, variants of the classical halving algorithm could obtain this guarantee. But the algorithm does not receive this feedback, which is why the more sophisticated algorithm of \cite{LobelLV17} is needed.}\ar{Surely in their paper the bound must depend on $||\phi||$. How is that reflected in our bound?}\cj{the paper assumes $||\phi|| \le 1$. I'll add this info up in the paragraph. Also, I think this means we need an assumption that $||flatten(G)|| \le 1$ which we should include in section 2}.

Our strategy will be to instantiate $k \choose 2$ copies of this distance estimator --- one for each pair of actions --- to produce guesses  $(\hat{d}^t_{i,j})^2$ intended to approximate the \emph{squared} pairwise distances $d(x_i^t, x_j^t)^2$. \cj{Should we be more consistent in terms of whether we use $\bar{d}_{i,j}^t$ or $d(x_i^t, x_j^t)$ } \cj{TODO: propagate the comma , for all the $d_{ij}^t$}From this we derive estimates $\hat{d}^t_{i,j}$ of the pairwise distances $d(x_i^t,x_j^t)$. Note that this is a linear estimation problem for any Mahalanobis distance, because by our observation in Section \ref{sec:mah}, a squared Mahalanobis distance can be written as a linear function of the $m = d^2$ unknown entries of the matrix $G = A^{\top} A$ which defines the Mahalanobis distance.

The complication is that the $\DE$ algorithms expect feedback at every round, which we cannot always provide. This is because the fairness oracle $\cO_d$ provides feedback about the distribution $\pi(\bar{r}^t, \hat{d}^t)$ used by the algorithm, \emph{not} directly about the guesses $\hat{d}^t$. These are not the same, because not all of the constraints in the linear program $LP(\bar{r}^t, \hat{d}^t)$ are necessarily tight --- it may be that $|\pi(\bar{r}^t, \hat{d}^t)_i - \pi(\bar{r}^t, \hat{d}^t)_j| < \hat{d}^t_{i,j}$. For any copy of $\DE$ that does not receive feedback, we can simply ``roll back'' its state and continue to the next round. But we need to argue that we make progress --- that whenever we are $\epsilon$-unfair, or whenever we experience large per-round regret, then there is at least one copy of $\DE$ that we can give feedback to such that the corresponding copy of $\DE$ has made a large prediction error, and we can thus charge either our fairness loss or our regret to the mistake bound of that copy of $\DE$.

As we show, there are three relevant cases.
\begin{enumerate}
\item In any round in which we are $\epsilon$-unfair for some pair of contexts $x^t_i$ and $x^t_j$, then it must be that $\hat{d}^t_{i,j} \geq d(x_i^t,x_j^t) + \epsilon$, and so we can always update the $(i,j)$th copy of $\DE$ and charge our fairness loss to its mistake bound. We formalize this in Lemma \ref{fairloss}.
\item For any pair of arms $(i,j)$ such that we have not violated the fairness constraint, \emph{and} the $(i,j)$th constraint in the linear program is tight, we can provide feedback to the $(i,j)$th copy of $\DE$ (its guess was not too large). There are two cases. Although the algorithm never knows which case it is in, we handle each case separately in the analysis.
\begin{enumerate}
\item For every constraint $(i,j)$ in $LP(\bar{r}^t, \hat{d}^t)$ that is \emph{tight} in the optimal solution, $|\hat{d}^t_{i,j}-d(x_i^t,x_j^t)| \leq \epsilon$. In this case, we show that our algorithm does not incur very much per round regret. We formalize this in Lemma \ref{close_obj_corollary}.
\item Otherwise, there is a tight constraint $(i,j)$ such that  $|\hat{d}^t_{i,j}-d(x_i^t,x_j^t)| > \epsilon$. In this case, we may incur high per-round regret --- but we can charge such rounds to the mistake bound of the $(i,j)$th copy of $\DE$ using Lemma \ref{fairloss}.
\end{enumerate}
\end{enumerate}

\subsection{The Distance Estimator}
\label{sec:de}
First, we fix some notation for the $\DE$ algorithm.
We write $\DE(\epsilon)$ to instantiate a copy of $\DE$ with a mistake bound for $\epsilon$-misestimations. The mistake bound we state for $\DE$ is predicated on the assumption that the norm of the unknown linear parameter vector $\alpha \in \mathbb{R}^m$ is bounded by $|| \alpha || \le B_1$, and the norms of the arriving vectors $u^t \in \mathbb{R}^m$ are bounded by $||u^t|| \le B_2$. Given an instantiation of $\DE$ and a new vector $u^t$ for which we would like a prediction, we write: $g^t = \DE.guess(u^t)$ for its guess of the value of  $\langle \alpha, u^t \rangle$. We use the following notation to refer to the feedback we provide to $\DE$: If $g^t > \langle \alpha, u^t \rangle$ and we provide feedback, we write $\DE.feedback(\top)$. Otherwise, if $g^t \leq  \langle \alpha, u^t \rangle$ and we give feedback, we write $\DE.feedback(\bot)$.  In some rounds, we may be unable to provide the feedback that $\DE$ is expecting: in these rounds, we simply ``roll-back'' its internal state. We can do this because the mistake bound for $\DE$ holds for \emph{every} sequence of arriving vectors $u^t$. If we give feedback to $\DE$ in a given round $t$, we write $v^t=1$ write $v^t=0$ otherwise.

\begin{definition}
Given an accuracy parameter $\epsilon$, a linear parameter vector $\alpha$, a sequence of vectors $u^1, \ldots, u^T$, a sequence of guesses $g^1, \ldots, g^T$ and a sequence of feedback indicators, $v^1, \ldots, v^T$, the number of valid $\epsilon$-mistakes made by $\DE$ is:
$$ \mathbf{Mistakes}(\epsilon)  = \sum_{t=1}^T \mathbbm{1}(v^t = 1 \land \vert g^t - \langle u^t, \alpha \rangle \vert > \epsilon)$$
In other words, it is the number of $\epsilon$-mistakes made by $\DE$ in rounds for which we provided the algorithm feedback.
\end{definition}

We now state a version of the main theorem from \cite{LobelLV17}, adapted to our setting\footnote{In \cite{LobelLV17}, the algorithm receives feedback in every round, and the scale parameters $B_1$ and $B_2$ are normalized to be $1$. But the version we state is an immediate consequence.}:

\begin{lemma}[\cite{LobelLV17}]
\label{distanceEstimatorRegret}
For any $\epsilon > 0$ and any sequence of vectors $u^1,\ldots,u^T$, $\DE(\epsilon)$ makes a bounded number of valid $\epsilon$-mistakes.
$$\mathbf{Mistakes}(\epsilon)  = O\left(m \log\left(\frac{m \cdot B_1 \cdot B_2}{\epsilon}\right)\right)$$
\end{lemma}

\subsection{The Algorithm}

\begin{algorithm2e}
\label{alg:known}
  \SetAlgoLined
  \caption{ $\algo_{\textrm{known}-\theta}$}
  \For{$i,j = 1, \hdots, k$}{
   	$\DE_{i,j} = \DE(\epsilon^2)$
   }
   \For{$t = 1, \hdots, T$}{
   receive the contexts $x^t = (x_1^t, \hdots, x_k^t)$\\
   \For{$i,j = 1, \hdots, k$}{
   	$u^t_{i,j} = flatten((x_i^t - x_j^t)(x_i^t-x_j^t)^\top)$\\
   	$g^t_{i,j} = \DE_{ij}.guess(u^t_{i,j})$\\
	$\hat{d}^t_{i,j} = \sqrt{g^t_{i,j}}$\\
   }
   $\pi^t = \pi({\bar{r}}^t, \hat{d}^t)$ \\
   Pull an arm $i^t$ according to $\pi^t$ and receive a reward $r^t_{i^t}$\\

   $S = \bm{O}_d(x^t, \pi^t)$\\
   $R = \{(i,j) | (i,j) \notin S \land |p^t_i - p^t_j| = \hat{d}_{ij}^t\}$\\
   \For{$(i,j) \in S$}{
   	$\DE_{ij}.feedback(\bot)$\\
	$v_{ij}^t = 1$
   }
   \For{$(i,j) \in R$}{
   	$\DE_{ij}.feedback(\top)$\\
	$v_{ij}^t = 1$
   }	
  }
  \label{alg:knownobj}
\end{algorithm2e}

For each pair of arms $i,j \in [k]$, our algorithm instantiates a copy of $\DE(\epsilon^2)$, which we denote by $\DE_{i,j}$: we also subscript all variables relevant to $\DE_{i,j}$ with $i,j$ (e.g. $u_{i,j}^t$). The underlying linear parameter vector we want to learn $\alpha = flatten(G) \in \mathbb{R}^{d^2}$, where $flatten: \mathbb{R}^{m \times n} \rightarrow \mathbb{R}^{m \cdot n}$ maps a $m \times n$ matrix to a vector of size $mn$ by concatenating its rows into a vector. Similarly, given a pair of contexts $x_i^t,x_j^t$, we will define $u^t_{i,j} = flatten((x_i^t - x_j^t)(x_i^t - x_j^t)^\top)$. $\DE_{i,j}.guess(u^t_{i,j})$ will output guess $g^t_{i,j}$ for the value $\langle \alpha, u^t_{i,j} \rangle = (\bar{d}^t_{i,j})^2$, as $$\langle flatten(G), flatten((x_i^t - x_j^t)(x_i^t - x_j^t)^\top) \rangle =\sum_{a,b=1}^d G_{a,b}(x_i^t-x_j^t)_a(x_i^t-x_j^t)_b=(\bar{d}^t_{i,j})^2$$ We take $\hat{d}^t_{i,j} = \sqrt{g^t_{i,j}}$ as our estimate for the distance between $x_i^t$ and $x_j^t$.

The algorithm then chooses an arm to pull according to the distribution $\pi(\bar{r}^t, \hat{d}^t)$, where $\bar{r}_i^t = \langle \theta, x_i \rangle$. The fairness oracle $O_d$ returns all pairs of arms that violate the fairness constraints. For these pairs $(i,j)$ we provide feedback to $\DE_{i,j}$: the guess was too large. For the remaining pairs of arms $(i,j)$, there are two cases. If the $(i,j)$th constraint in $LP(\bar{r}^t, \hat{d}^t)$ was not tight, then we provide no feedback ($v^t_{i,j} = 0)$. Otherwise, we provide feedback: the guess was not too large. The pseudocode appears as Algorithm \ref{alg:knownobj}.
% that were tight on the fairness constraints with $\hat{d}$ (e.g. $|\pi^t_i - \pi^t_j| = \hat{d}^t_{i,j}$), because our estimate of contextual distance was not greater than the true value, provide $\top$ feedback to inform that $g^t_{i,j} \le \langle \alpha, u^t_{i,j} \rangle$. For all the other pairs $(i,j)$ that we couldn't provide a feedback to, mark them as invalid rounds. \\

First we derive the valid mistake bound that the $\DE_{i,j}$ algorithms incur in our parameterization.

\begin{lemma}
\label{distanceEstimator}
For pair $(i,j)$, the total number of valid $\epsilon^2$ mistakes made by $\DE_{i,j}$ is bounded as:
$$\mathbf{Mistakes}(\epsilon^2)  = O\left(d^2 \log\left(\frac{d \cdot ||A^\top A||_F }{\epsilon}\right)\right)$$
where the distance function is defined as $d(x_i,x_j) = ||Ax_i - Ax_j||_2$ and $||\cdot||_F$ denotes the Frobenius norm.
\end{lemma}
\begin{proof}
This follows directly from Lemma \ref{distanceEstimatorRegret}, and the observations that in our setting, $m = d^2$, $B_1 = ||\alpha|| = ||A^\top A||_F$, and
$$B_2 \leq \max_t ||u^t_{i,j}||_2 \leq \max_t ||(x_i^t - x_j^t)||^2 \leq 4.$$
\end{proof}

We next observe that since we only instantiate $k^2$ copies of $\DE$ in total, Lemma \ref{distanceEstimator} immediately implies the following bound on the total number of rounds in which \emph{any} distance estimator that receives feedback provides us with a distance estimate that differs by more than $\epsilon$ from the correct value:
\begin{corollary}
\label{de_corollary}
The number of rounds where there exists a pair $(i,j)$ such that feedback is provided ($v^t_{i,j}=1$) and its estimate is off by more than $\epsilon$ is bounded:
$$\left|\{t: \exists (i,j): v_{ij}^t=1 \wedge \vert \hat{d}^t_{i,j} - \bar{d}^t_{i,j} \vert > \epsilon\}\right| \le O\left(k^2d^2 \log\left(\frac{d \cdot ||A^\top A||_F }{\epsilon}\right)\right)$$
\end{corollary}
\begin{proof}
This follows from summing the $k^2$ valid $\epsilon^2$ mistake bounds for each copy of $\DE_{i,j}$, and noting that an $\epsilon$ mistake in predicting the value of $\bar{d}^t_{i,j}$ implies an $\epsilon^2$ mistake in predicting the value of $(\bar{d}^t_{i,j})^2$.
\end{proof}

%This gives us all the necessary tool to state the algorithm now. We first use $distanceEstimator_{ij}$ to guess $\hat{d}_{ij}^t$ and obtain $\pi({\bar{r}}^t, \hat{d}^t)$ by solving the linear programming described above. Then, by querying the fairness oracle with the contexts $x^t$ and the produced probability distribution $\pi^t$, we obtain all the pairs that have violated the fairness constraints. Fairness violation could have only happened because our contextual distances were too loose, so we provide $\bot$ to these set of pairs. Then, for all the other pairs that didn't violate fairness constraints and were tight with respect to our estimated contextual distances, we provide $\top$, as we could have relaxed these constraints and possibly increased the objective function of the linear programming. We are only interested in tight constraints, as relaxing non-tight constraints won't increase the objective function, as it was already loose. \\%

We now have the pieces to bound the $\epsilon$-unfairness loss of our algorithm:

\begin{theorem}
For any sequence of contexts and any Mahalanobis distance $d(x_1,x_2) = ||Ax_1-Ax_2||_2$:
\label{fairloss}
$$\mathbf{FairnessLoss}(\algo_{\textrm{known}-\theta}, T, \epsilon) \leq   O\left(k^2d^2 \log\left(\frac{d \cdot ||A^TA||_F }{\epsilon}\right)\right)$$
\end{theorem}
\begin{proof}
\begin{align*}
\mathbf{FairnessLoss}(\algo_{\textrm{known}-\theta}, T, \epsilon) &= \sum_{t=1}^T \mathbf{Unfair}(\algo_{\textrm{known}-\theta},\epsilon)\\
&\leq \sum_{t=1}^T \sum_{i,j} \mathbbm{1}( |\pi^t_i-\pi^t_j| > \bar{d}_{ij}^t + \epsilon)\\
&= \sum_{i,j}\sum_{t=1}^T \mathbbm{1}(\{v_{ij}^t = 1 \wedge \hat{d}_{ij}^t > d^t_{ij} + \epsilon\})\\
&\le \sum_{i,j}\sum_{t=1}^T \mathbbm{1}(\{v_{ij}^t = 1 \wedge \vert\hat{d}_{ij}^t - d^t_{ij}\vert > \epsilon\})\\
&=  O\left(k^2d^2 \log\left(\frac{d \cdot ||A^\top A||_F }{\epsilon}\right)\right) &\text{Corollary \ref{de_corollary}}\\
\end{align*}
%Fourth equality follows from the fact that with the linear programming constraint $\vert \pi_i^t - \pi_j^t \vert \le \hat{d}^t_{ij}$, $|\pi^t_i - %\pi^t_j| > \bar{d}_{ij}^t$ could have only happened if $\hat{d}_{ij}^t > \bar{d}^t_{ij} + \epsilon$.\\
\end{proof}

We now turn our attention to bounding the regret of the algorithm. Recall from the overview in Section \ref{sec:known-overview}, that our plan will be to divide rounds into two types. In rounds of the first type, our distance estimates corresponding to every \emph{tight constraint} in the linear program have only small error. We cannot bound the number of such rounds, but we can bound the regret incurred in any such rounds. In rounds of the second type, we have at least one significant error in the distance estimate corresponding to a tight constraint. We might incur significant regret in such rounds, but we can bound the number of such rounds.

The following lemma bounds the \emph{decrease} in expected per-round reward that results from under-estimating a \emph{single} distance constraint in our linear programming formulation.

\begin{lemma}
\label{close_obj}
Fix any vector of distance estimates $d$ and any vector of rewards $r$. Fix a constant $\epsilon$ and any pair of coordinates $(a,b) \in [k] \times [k]$. Let $d'$ be the vector such that $d'_{ab} = d_{ab} - \epsilon$ and $d'_{ij} = d_{ij}$ for $(i,j) \neq (a,b)$, then $\langle r, \pi(r, d) \rangle- \langle r, \pi(r, d') \rangle \le \epsilon \sum_{i=1}^k r_i $
\end{lemma}
\begin{proof}
The plan of the proof is to start with $\pi(r,d)$ and perform surgery on it to arrive at a new probability distribution $p' \in \Delta k$ that satisfies the constraints of $LP(r, d')$, and obtains objective value at least $\langle r, p' \rangle \geq  \langle r, \pi(r, d) \rangle -  \epsilon \sum_{i=1}^k r_i$. Because $p'$ is feasible, it lower bounds the objective value of the optimal solution $\pi(r, d')$, which yields the theorem.

To reduce notational clutter, for the rest of the argument we write $p$ to denote $\pi(r, d)$. Without loss of generality, we assume that $p_a \ge p_b$. If $p_a - p_b \le d_{ab} - \epsilon$, then $p_i$ is still a feasible solution to $LP(r,d')$, and we are done. Thus, for the rest of the argument, we can assume that $p_a - p_b > d_{ab} - \epsilon$. We write  $\Delta = (p_a - p_b) - (d_{ab} - \epsilon) > 0$

%As a result of the tighter constraint $d'_{ab} = d_{ab} - \epsilon$, we will decrease $p_a$ by $\Delta = (p_a - p_b) - (d_{ab} - \epsilon) > 0$ so that $p'_a - p'_b = d_{ab} - \epsilon$. Also, along with $a$, we'll decrease all the probabilities whose value is greater than $p_a$ by $\Delta$, too. Lastly, we'll update probabilities, $p_i$ where $p_b + (d_{ij} - \epsilon) \le p_i < p_a$ to be $p_b + (d_{ij} - \epsilon)$. \\

We now define our modified distribution $p'$:
\[
p'_i =
     \begin{cases}
       p_i - \Delta & p_a \le p_i \\
       p_a - \Delta  & p_a - \Delta \le p_i < p_a  \\
       p_i & \text{otherwise }\\
       \end{cases}
\]

\begin{figure}
\label{fig:M1}
\centering
\begin{tikzpicture}

\draw[thick,->] (0,0) -- (4.5,0) node[anchor=north west] {sorted};
\draw[black]   plot[smooth,domain=0:4] (\x, {1/4*\x*\x}) node[below right] {$p$};
\draw[dotted, very thick]   plot[smooth,domain=2:4] (\x, {1/4*\x*\x-0.5}) node[below right] {$p'$};
\draw[dotted, very thick] (1.414,0.5) -- (2,0.5);
\draw[dotted, very thick]   plot[smooth,domain=0:1.414] (\x, {1/4*\x*\x});
\draw[thick,->] (0,0) -- (0,4) node[anchor=south east] {probability};
\draw (2,1) circle[radius=2pt] node[above] {$p_a$};
\draw (2,0.5) circle[radius=2pt] node[below] {$p_a - \Delta$};
\end{tikzpicture}
\caption{A visual interpretation of the surgery performed on $p$ in the proof of Lemma \ref{close_obj} to obtain $P'$. Note that the surgery manages to shrink the distance between $p_a$ and $p_b$ without increasing the distance between any other pair of points. \ar{Might be useful to have $p_b$ in this diagram}\cj{yeah, but then it's actually not clear where $p_b$ should be on the diagram, unless I assume $\Delta=\epsilon$. I think this should be good}}
\end{figure}

We'll partition the coordinates of $p_i$ into which of the three cases they fall into in our definition of $p'$ above.  $S_1 = \{i | p_a \le p_i\}$, $S_2 = \{i | p_a -  \epsilon \le p_i < p_a \}$, and $S_3 = \{i | i < p_b + (d_{ab} - \epsilon)\}$. It remains to verify that $p'$ is a feasible solution to $LP(r, d')$, and that it obtains the claimed objective value.

\paragraph{Feasibility:}
First, observe that $\sum_i p'_i \leq 1$. This follows because $p'$ is coordinate-wise smaller than $p$, and by assumption, $p$ was feasible. Thus, $\sum_i p'_i \leq \sum_i p_i \leq 1$.

Next, observe that by construction, $p'_i \geq 0$ for all $i$. To see this, first observe that $p_a - \Delta = p_b + (d_{ab} - \epsilon) \ge 0$ where the last inequality follows because $d_{ab} \ge \epsilon$. We then consider the three cases:
\begin{enumerate}
\item For $i \in S_1$, $p'_i = p_i -\Delta \ge p_a - \Delta \ge 0$ because $p_i \ge p_a$.
\item For $i \in S_2$, $p'_i =p_a - \Delta \ge 0$.
\item For $i \in S_3$, $p'_i = p_i \ge 0$.
\end{enumerate}

Finally, we verify that for all $(i,j)$, $\vert p'_i - p'_j \vert \le d'_{ij}$. First, observe that  $p'_a - p'_b = (p_b + (d_{ab} - \epsilon)) - p'_b = d_{ab} - \epsilon = d'_{ab}$, and so the inequality is satisfied for index pair $(a,b)$. For all the other pairs $(i,j) \neq (a,b)$, we have  $d'_{ij} = d_{ij}$, so it is enough to show that $\vert p'_i - p'_j \vert \le d_{ij}$. Note that for all $x,y \in \{1,2,3\}$ with $x < y$, if $i \in S_x$ and $j \in S_y$, we have that $x \leq y$. Therefore, it is sufficient to verify the following six cases:
\begin{enumerate}
  \item $i \in S_1, j \in S_1$:
  $\vert p'_i - p'_j \vert = (p_i - \Delta) - (p_j  - \Delta) = p_i - p_j \le d_{ij}$

  \item $i \in S_1, j \in S_2$:
  $\vert p'_i - p'_j \vert= (p_i - \Delta) - (p_a  - \Delta) = p_i - p_a < p_i - p_j \le d_{ij}$

  \item $i \in S_1, j \in S_3$:
  $\vert p'_i - p'_j \vert = (p_i - \Delta) - p_j = (p_i - p_j) -\Delta \le (p_i - p_j) \le d_{ij} $

  \item $i \in S_2, j \in S_2$:
  $\vert p'_i - p'_j \vert = (p_a - \Delta) - (p_a - \Delta) = 0 \le d_{ij}$

  \item $i \in S_2, j \in S_3$:
  $\vert p'_i - p'_j \vert = (p_a - \Delta) - p_j \le p_i - p_j \le d_{ij}$

  \item $i \in S_3, j \in S_3$:
  $\vert p'_i - p'_j \vert = p_i - p_j \le d_{ij}$
\end{enumerate}
Thus, we have shown that $p'$ is a feasible solution to $LP(r,d')$.

\paragraph{Objective Value:} Note that for each index $i$, $p_i - p'_i \le \Delta \le \epsilon$. Therefore we have:
\begin{align*}
\langle r, \pi(r, d) \rangle- \langle r, \pi(r, d') \rangle &\le \langle r, \pi(r, d) \rangle- \langle r, p' \rangle \\
&= \langle r, p - p'\rangle\\
&\le \epsilon \sum_{i=1}^k r_i
\end{align*}
which completes the proof.
\end{proof}

We now prove the main technical lemma of this section. It states that in any round in which the error of our distance estimates for \emph{tight constraints} is small (even if we have high error in the distance estimates for slack constraints), then we will have low per-round regret.

\begin{lemma}
\label{close_obj_corollary}
At round $t$, if for all pairs of indices $(i,j)$, we have either:
\begin{enumerate}
\item $|\hat{d}_{i,j}^t - \bar{d}_{i,j}^t| \le \epsilon$ or
\item $v_{i,j}^t=0$ (corresponding to an LP constraint that is not tight)
\end{enumerate}
then:
$$\langle r^t, \pi(r^t,\bar{d}^t) \rangle - \langle r^t, \pi(r^t, \hat{d}^t) \rangle \le \epsilon k^3$$
for any vector $r^t$ with $||r^t||_\infty \leq 1$.
\end{lemma}
\begin{proof}
First, define $\tilde{d}^t$ to be the coordinate-wise maximum of $\hat{d}^t$ and $\bar{d}^t$: i.e. the vector such that for every pair of coordinates $i,j$, $\tilde{d}_{ij} = \max(\bar{d}_{ij},\hat{d}_{ij})$. To simplify notation, we will write $\hat{p} =  \pi(r^t, \hat{d}^t)$, $\bar{p} =  \pi(r^t, \bar{d}^t)$, and $\tilde{p} =  \pi(r^t, \tilde{d}^t)$.

We make three relevant observations:
\begin{enumerate}
\item  First, because $LP(r^t, \tilde{d}^t)$ is a relaxation of $LP(r^t, \bar{d}^t)$, it has only larger objective value. In other words, we have that $\langle r^t, \tilde{p} \rangle \geq \langle r^t, \bar{p} \rangle$. Thus, it suffices to prove that $\langle r^t, \hat{p} \rangle \geq \langle r^t, \tilde{p} \rangle  - \epsilon k^3$.
    \item Second, for all pairs $i,j$, $|\hat{d}_{i,j}^t - \tilde{d}_{i,j}^t| \leq |\hat{d}_{i,j}^t - \bar{d}_{i,j}^t|$. Thus, if we had $|\hat{d}_{i,j}^t - \bar{d}_{i,j}^t| \le \epsilon$, we also have $|\hat{d}_{i,j}^t - \tilde{d}_{i,j}^t| \le \epsilon$.
    \item Finally, by construction, for every pair $(i,j)$, we have $\tilde{d}_{ij} \geq \hat{d}_{ij}$
\end{enumerate}

Let $S_1$ be the set of indices $(i,j)$ such that $|\hat{d}_{i,j}^t - \tilde{d}_{i,j}^t| \le \epsilon$, and let $S_2$ be the set of indices $(i,j)\not\in S_1$ such that $v_{i,j}^t=0$. Note that by assumption, these partition the space, and that by construction, for every $(i,j) \in S_2$, the corresponding constraint in $LP(r^t, \hat{d}^t)$ is not tight: i.e. $|\hat{p}_i - \hat{p}_j| < \hat{d}^t_{i,j}$. Let $d^*$ be the vector such that for all $(i,j) \in S_1$, $d^*_{ij} = \hat{d}_{ij}$, and for all $(i,j) \in S_2$, $d^*_{ij} = \tilde{d}_{ij}$. Observe that $LP(r^t, d^*)$ corresponds to a relaxation of $LP(r^t, \hat{d})$ in which \emph{only constraints that were already slack were relaxed}. As a result, $\hat{p}$ is also an optimal solution to $LP(r^t, d^*)$. Note also that by construction, we now have that for \emph{every} pair $(i,j)$:  $|\tilde{d}_{ij} - d^*_{ij}| \le \epsilon$

Our argument will proceed by describing a sequence of $n+1 = k^2+1$ vectors $p^0,p^1,\ldots,p^n$ such that $p^0 = \tilde{p}$, $p^n$ is a feasible solution to $LP(r^t,d^*)$, and for all adjacent pairs $p^{\ell}, p^{\ell+1}$, we have: $\langle r^t, p^{\ell+1} \rangle \geq \langle r^t, p^{\ell} \rangle - \epsilon k$. Telescoping these inequalities yields:
$$\langle r^t, \hat{p} \rangle \geq  \langle r^t, p^n \rangle \geq  \langle r^t, \tilde p \rangle - k^3 \epsilon $$
which will complete the proof.

To finish the argument, fix an arbitrary ordering on the indices $(i,j) \in [k] \times [k]$, which we denote by $(i_1,j_1),\ldots,(i_n,j_n)$. Define the distance vector $d^\ell$ such that:
 $$d^{\ell}_{i_a,j_a} = \left\{
                            \begin{array}{ll}
                              \tilde d_{i_a,j_a}, & \hbox{If $a > \ell$;} \\
                              d^*_{i_a,j_a}, & \hbox{If $a \leq \ell$.}
                            \end{array}
                          \right.$$
Note that the sequence of distance vectors $d^1,\ldots,d^n$ ``walks between'' $\tilde d$ and $d^*$ one coordinate at a time. Now let $p^{\ell} = \pi(r^t, d^\ell)$. By construction, we have that every pair $(d^{\ell}, d^{\ell+1})$ differ in only a single coordinate, and that the difference has magnitude at most $\epsilon$. Therefore, we can apply Lemma \ref{close_obj} to conclude that:
$$\langle r^t, p^{\ell+1} \rangle \geq \langle r^t, p^{\ell} \rangle - \epsilon \sum_{i=1}^k r^t_i \geq \langle r^t, p^{\ell} \rangle - \epsilon k$$
as desired.
\end{proof}

Finally, we have all the pieces we need to prove a regret bound for  $\algo_{\textrm{known}-\theta}$.

\begin{theorem}
\label{thm:known-regret}
For any time horizon $T$:
$$ \mathbf{Regret}(\algo_{\textrm{known}-\theta},T) \le O\left(k^2d^2 \log\left(\frac{d \cdot ||A^\top A||_F }{\epsilon}\right) + k^3\epsilon T\right) $$
\cj{This is a little different in its notation from the preliminary. Should be consistent. introduced this short-hand version in the preliminary}
Setting $\epsilon = O(1/(k^3 T))$ yields a regret bound of $O(d^2\log(||A^\top A||_F\cdot dkT))$.
\end{theorem}
\begin{proof}
We partition the rounds $t$ into two types. Let $S_1$ denote the rounds such that there is at least one pair of indices $(i,j)$ such that one instance $\DE_{ij}$ produced an estimate that had error more than $\epsilon$, and it was provided feedback. We let $S_2$ denote the remaining rounds, for which for \emph{every} pair of indices $(i,j)$, \emph{either} $\DE_{ij}$ produced an estimate that had error at most $\epsilon$, or $\DE_{ij}$ was not given feedback.
$$S_1 = \{t : \exists (i,j) : |\hat{d}^t_{ij} - \bar{d}^t_{ij}| > \epsilon \ \mathrm{and} \ v^t_{ij} = 1\} \ \ \ S_2 = \{t : \forall (i,j) : |\hat{d}^t_{ij} - \bar{d}^t_{ij}| \leq \epsilon \ \mathrm{or} \ v^t_{ij} = 0\}$$

Observe that $S_1$ and $S_2$ partition the set of all rounds. Next, observe that Corollary \ref{de_corollary} tells us that:
$$|S_1| \leq  O\left(k^2d^2 \log\left(\frac{d \cdot ||A^\top A||_F }{\epsilon}\right)\right)$$
and Lemma \ref{close_obj_corollary} tells us that for every round $t \in S_2$, the per-round regret is at most $\epsilon k^3$. Together with the facts that $|S_2| \leq T$ and that the per-round regret for any $t \in S_1$ is at most $1$, we obtain:
$$ \mathbf{Regret}(\algo_{\textrm{known}-\theta},T) \le O\left(k^2d^2 \log\left(\frac{d \cdot ||A^\top A||_F }{\epsilon}\right) + k^3\epsilon T\right) $$

\end{proof}

\section{The Full Algorithm}
\label{sec:full}

In this section, we present our final algorithm, which has no knowledge of either the distance function $d$ or the linear objective $\theta$. The resulting algorithm shares many similarities with the algorithm we developed in Section \ref{sec:knownobj}, and so much of the analysis can be reused.

\subsection{Outline of the Solution}
At a high level, our plan will be to combine the techniques we developed in Section \ref{sec:knownobj} with a standard ``optimism in the face of uncertainty'' strategy for learning the parameter vector $\theta$. Our algorithm will maintain a ridge-regression estimate $\tilde \theta$ together with confidence regions derived in \cite{Abbasi-YadkoriPS11}. After it observes the contexts $x_i^t$ at round $t$, it uses these to derive upper confidence bounds on the expected rewards, corresponding to each context --- represented as a vector $\hat{r}^t$. The algorithm continues to maintain distance estimates $\hat{d}^t$ using the $\DE$ subroutines, identically to how they were used in Section \ref{sec:knownobj}. At ever round, the algorithm then chooses its action according to the distribution $\pi^t = \pi(\hat{r}^t, \hat{d}^t)$.

The regret analysis of the algorithm follows by decomposing the per-round regret into two pieces. The first can be bounded by the sum of the \emph{expected widths} of the confidence intervals corresponding to each context $x_i^t$ that might be chosen at each round $t$, where the expectation is over the randomness of the algorithm's distribution $\pi^t$. A theorem of \cite{Abbasi-YadkoriPS11} bounds the sum of the widths of the confidence intervals corresponding to arms \emph{actually chosen} by the algorithm (Lemma \ref{sum_ci}). Using a martingale concentration inequality, we are able to relate these two quantities (Lemma \ref{ci_converg}). We show that the second piece of the regret bound can be manipulated into a form that can be bounded using Lemmas \ref{fairloss} and \ref{close_obj_corollary} from Section \ref{sec:knownobj} (Theorem \ref{thm:lastregret}).

\subsection{Confidence Intervals from \cite{Abbasi-YadkoriPS11}}
We would like to be able to construct confidence intervals at each round $t$ around each arm's expected reward such that for each arm $i$, with probability $1-\delta$, $\bar{r}_i^t \in [\tilde{r}_i^t+w_i^t, \tilde{r}_i^t+w_i^t]$, where $\tilde{r}^t_i$ is our ridge-regression estimate of $\bar{r}^t_i$ and $w^t_i$ is the confidence interval width around the estimate. Our algorithm will make use of such confidence intervals for the ridge regression estimator derived and analyzed in \cite{Abbasi-YadkoriPS11}, which we recount here.

Let $\tilde{V}^t = {X^t}^\top X^t + \lambda I$ be a regularized design matrix, where $X^t = [x^1_{i_1}, \hdots, x^{t-1}_{i_{t-1}}]$ represents all the contexts whose rewards we have observed up to but not including time $t$. Let $Y^t = [r^1_{i_1}, \hdots, r^{t-1}_{i_{t-1}}]$ be the corresponding vector of observed rewards. $\tilde{\theta} = (V^t)^{-1}{X^t}^\top Y^t$ is the (ridge regression) regularized least squares estimator we use at time $t$. We write $\tilde{r}^t_i = \langle \tilde{\theta}, x^t_i \rangle$ for the reward point prediction that this estimator makes at time $t$ for arm $i$. \ar{Changing from $\tilde \theta$ to $\tilde \theta$ since the linear parameter vector is called $\theta$. }

We can construct the following confidence intervals around $\tilde{r}^t$:
\begin{lemma}[\cite{Abbasi-YadkoriPS11}]\label{ci_lemma}
With probability $1-\delta$,
$$|\bar{r}^t_i - \tilde{r}^t_i |= |\langle x^t_i, (\theta-\tilde{\theta}) \rangle| \le \| x_{i}^t\|_{(\bar{V^t})^{-1}} \Big( \sqrt{2d \log(\frac{1+t/\lambda}{\delta})} + \sqrt{\lambda}\Big)$$
where $||x||_{A} = \sqrt{x^\top A x}$
\end{lemma}
Therefore, the confidence interval widths we use in our algorithm will be
$$w^t_i = \min(\| x_{i}^t\|_{(\bar{V^t})^{-1}} \Big( \sqrt{2d \log(\frac{1+t/\lambda}{\delta})} + \sqrt{\lambda}\Big), 1)$$ (expected rewards are bounded by $1$ in our setting, and so the minimum maintains the validity of the confidence intervals).  The upper confidence bounds we use to compute our distribution over arms will be $\hat{r}^t_i = \tilde{r}^t_i + w^t_i$. We will write $w^t = [w^t_1, \hdots, w^t_k]$ to denote the vector of confidence interval widths at round $t$.

Little can be said about the widths of these confidence intervals in isolation. However, the following theorem bounds the \emph{sum} (over time) of the widths of the confidence intervals around the contexts actually selected.
\begin{lemma}[\cite{Abbasi-YadkoriPS11}]
\label{sum_ci}
$$  \sum_{t=1}^T w_{i^t}^t \le \sqrt{2d \log\big(1+ \frac{T}{d\lambda} \big)} \Big(\sqrt{2dT \log(\frac{1+T/\lambda}{\delta})} + \sqrt{T\lambda} \Big)$$
\end{lemma}

\subsection{The Algorithm}
The pseudocode for the full algorithm is given in Algorithm \ref{alg:full}.

\begin{algorithm2e}
\label{alg:full}

  \caption{$\algo_\textrm{full}$}
   \For{$i,j = 1, \hdots, k$}{
   	$\DE_{ij} = \DE(\epsilon^2)$
   }
   \For{$t = 1, \hdots, T$}{
   receive the contexts $x^t = (x_1^t, \hdots, x_k^t)$\\
   $X^t = [x^1, \hdots, x^{t-1}]$  \\
   $Y^t = [r^t, \hdots, r^{t-1}]$ \\
   $\tilde{V}^t = {X^t}^\top X^t + \lambda I$\\
   $\tilde{\theta} = (V^t)^{-1}{X^t}^\top Y^t$\\
   \For{$i = 1, \hdots, k$}{
   	$\tilde{r}^t_i = \langle \tilde{\theta}, x^t_i\rangle$\\
	  $w^t_i = \min\left(\| x_{i}^t\|_{(\bar{V^t})^{-1}} \Big( \sqrt{2d \log(\frac{1+t/\lambda}{\delta})} + \sqrt{\lambda}\Big),1\right)$\\
	 $\hat{r}^t_i= \tilde{r}^t_i + w^t_i$
   }
   \For{$i,j = 1, \hdots, k$}{
   	$u^t_{i,j} = flatten((x_i^t - x_j^t)(x_i^t-x_j^t)^T))$\\
   	$g^t_{i,j} = \DE_{i,j}.guess(u^t_{i,j})$\\
	$\hat{d^t_{ij}} = \sqrt{g^t_{i,j}}$ \\
   }
   $\pi^t = \pi(\hat{r}^t, \hat{d}^t)$\\	
   Pull an arm $i^t$ according to $\pi^t$ and receive a reward $r^t_{i^{t}}$\\
   $S = \bm{O}_d(x^t, \pi^t)$\\
   $R = \{(i,j) | (i,j) \notin S \land |\pi^t_i - \pi^t_j| = \hat{d}_{i,j}^t\}$\\
   \For{$(i,j) \in S$}{
   	$\DE_{i,j}.feedback(\bot)$\\
	$v_{i,j}^t = 1$
   }
   \For{$(i,j) \in R$}{
   	$\DE_{i,j}.feedback(\top)$\\
	$v_{i,j}^t = 1$
   }
  }
\end{algorithm2e}

%At each round $t$, for each arm $i$, we build an upper confidence $\hat{r}^t_i$ for its expected reward $\bar{r}^t_i$. Also, we use $\DE$ to estimate the contextual distance $\hat{d}^t_{i,j}$ between arm $i$ and $j$ like we did in section \ref{sec:knownobj}. We then simply pull an arm according to probability distribution $\pi(\hat{r}, \hat{d})$.
In our proof of Theorem \ref{thm:lastregret}, we will connect the regret of $\algo_{full}$ to the sum of the \emph{expected} widths of the confidence intervals pulled at each round. In contrast, what is bounded by Lemma \ref{sum_ci} is the sum of the \emph{realized} widths. Using the Azuma Hoeffding inequality, we can relate these two quantities.

\begin{lemma}[Azuma-Hoeffding inequality (\cite{hoeffding1963probability})]
Suppose $\{X_k: k= 0, 1, 2, 3, \hdots \}$ is a martingale and $$\abs{X_k - X_{k-1}} < c_k.$$ Then, for all positive integers $N$ and all positive reals $t$,

$$\Pr(X_N - X_0 \ge t) \le \exp(\frac{t^2}{2\sum_{k=1}^N c_k^2})$$
\end{lemma}

\begin{lemma}
\label{ci_converg}
$$\Pr\left(\sum_{t=1}^T \mathbb{E}_{i \sim \pi^t} [ w^t_i ] - \sum_{t=1}^T w^t_{i^t} \geq \sqrt{2T \log\frac{1}{\delta}}\right) \le \delta $$
\end{lemma}

\begin{proof}
Once $x^1, \hdots, x^{t-1}, r^1_{i^t}, \dots, r^{t-1}_{i^{t-1}}$ and $x^t$ are fixed, $\pi^t$ is fixed.
In other words, for the filtration $\mathscr{F}^t = \sigma(x^1, \hdots, x^{t-1},r^1_{i^t}, \dots, r^{t-1}_{i^{t-1}}, x^t)$, $w^t_{i^t}$ is $\mathscr{F}^t$ measurable.
 Now, define $$D^t = \sum_{s=1}^t \mathbb{E}_{i \sim \pi^{s}}[w_{i}^s] - \sum_{s=1}^t w^s_{i^s}$$
 with respect to $\mathscr{F}^t$. One can think of $D^t$ as the accumulated difference between the confidence width of the arm that was actually pulled and the expected confidence width.  It's easy to see that $\{D^t\}$ is a martingale, as $\mathbb{E}[D^1] = 0$, and $\mathbb{E}[D^{t+1} | \mathscr{F}^t] = D^t$.
 %, as $\mathbb{E}[w_{i^{t+1}}^{t+1} - \mathbb{E}_{i \sim \pi^{t+1}}[w_i^{t+1}] | \mathscr{F}^t] = 0$.

 Also, $D_{t} - D_{t-1} = w_{i^t}^{t} - \mathbb{E}_{i \sim \pi^t}[w_i^{t}] \le 1$, since the confidence interval widths are bounded by $1$. \ar{Notation collision: our matrices are called $A$. Lets change this to something else.}\cj{done.}

Applying the Azuma-Hoeffding inequality gives us the following:

\begin{align*}
\Pr(\sum_{t=1}^T \mathbb{E}_{i \sim \pi^t}[w^t_i] - \sum_{t=1}^T w^t_{i^t} \ge \epsilon)  = \Pr(D^T \ge \epsilon ) \le \exp(\frac{-\epsilon^2}{2T}) \end{align*}

Now, setting $\epsilon = \sqrt{2T \ln\frac{1}{\delta}}$ yields:

$$\Pr(\sum_{t=1}^T \mathbb{E}_{i \sim \pi^t} [ w^t_{i} ] - \sum_{t=1}^T w^t_{i^t} \geq \sqrt{2T \log\frac{1}{\delta}}) \le \delta $$
\end{proof}

\begin{theorem}
\label{thm:full-regret}
For any time horizon $T$, with probability $1-\delta$:
$$\mathbf{Regret}(\algo_{full},T) \leq O\left(k^2d^2 \log\left(\frac{d \cdot ||A^\top A||_F }{\epsilon}\right) + k^3\epsilon T + d\sqrt{T}\log(\frac{T}{\delta})\right)$$
If $\epsilon = 1/{k^3 T}$, this is a regret bound of  $O\left(k^2d^2 \log\left(kdT \cdot ||A^\top A||_F \right) + d\sqrt{T}\log(\frac{T}{\delta})\right)$
\label{thm:lastregret}
 \end{theorem}
 \begin{proof}
We can compute:
\begin{align*}
%&\mathbf{Regret}(\algo_{full}, T) \\
\mathbf{Regret}(\algo_{full}, T) &= \sum_{t=1}^T \E_{i \sim \pi(\bar{r}^t,\bar{d}^t)} [\bar{r}^t_i] - \sum_{t=1}^T \E_{i \sim \pi(\hat{r}^t, \hat{d}^t)} [\bar{r}^t_i] \\
&= \sum_{t=1}^T \langle \bar{r}^t, \pi(\bar{r}^t,\bar{d}^t) \rangle - \langle \bar{r}^t, \pi(\hat{r}^t, \hat{d}^t) \rangle \\
&= \sum_{t=1}^T \langle \bar{r}^t, \pi(\bar{r}^t,\bar{d}^t) \rangle  -\langle {\bar{r}}^t , \pi(\hat{r}^t, \bar{d}^t)\rangle + \langle {\bar{r}}^t , \pi(\hat{r}^t, \bar{d}^t)\rangle - \langle \bar{r}^t, \pi(\hat{r}^t, \hat{d}^t) \rangle \\
&\le \sum_{t=1}^T \langle \hat{r}^t, \pi(\hat{r}^t,\bar{d}^t) \rangle  -\langle {\bar{r}}^t , \pi(\hat{r}^t, \bar{d}^t)\rangle + \langle {\bar{r}}^t , \pi(\hat{r}^t, \bar{d}^t)\rangle - \langle \bar{r}^t, \pi(\hat{r}^t, \hat{d}^t) \rangle \\
&\le \sum_{t=1}^T \langle 2w^t, \pi(\hat{r}^t,\bar{d}^t) \rangle  + \langle {\bar{r}}^t , \pi(\hat{r}^t, \bar{d}^t)\rangle - \langle \bar{r}^t, \pi(\hat{r}^t, \hat{d}^t) \rangle \\
\end{align*}
Here, the first inequality follows from the fact that $\hat{r}^t$ is coordinate-wise larger than $\bar{r}^t$, and that  $\pi(\hat{r}^t,\bar{d}^t)$ is the optimal solution to $LP(\hat{r}^t,\bar{d}^t)$. The second inequality follows from $\bar{r} \in [\tilde{r} - w, \tilde{r} + w] = [\hat{r}-2w, \hat{r}]$. \ar{Actually, should it be an equality? And why the factor of 2?}

Just as in the proof of Theorem \ref{thm:known-regret}, we now partition time into two sets:
$$S_1 = \{t : \exists (i,j) : |\hat{d}^t_{ij} - \bar{d}^t_{ij}| > \epsilon \ \mathrm{and} \ v^t_{ij} = 1\} \ \ \ S_2 = \{t : \forall (i,j) : |\hat{d}^t_{ij} - \bar{d}^t_{ij}| \leq \epsilon \ \mathrm{or} \ v^t_{ij} = 0\}$$
Recall that corollary \ref{de_corollary} bounds $|S_1| \leq  O\left(k^2d^2 \log\left(\frac{d \cdot ||A^\top A||_F }{\epsilon}\right)\right)$. Since the per-step regret of our algorithm can be at most $1$, this means that rounds $t \in S_1$ can contribute in total at most $C \doteq O\left(k^2d^2 \log\left(\frac{d \cdot ||A^\top A||_F }{\epsilon}\right)\right)$ regret. Thus, for the rest of our analysis, we can focus on rounds $t \in S_2$.

%Now, we will just charge maximum regret on rounds where there exists at least one pair $(i,j)$ such that $|\hat{d}^t_{i,j} - \bar{d}^t_{i,j}| > \epsilon \land v^t_{i,j}=1$. Using corollary \ref{de_corollary},

%$$
%\mathbf{Regret}(\algo_{full}) \le C + \sum_{t:\forall \{i,j\} \ldotp (|\hat{d}^t_{i,j} - \bar{d}^t_{ij}| \le \epsilon) \lor v_{i,j}^t=0} \langle 2w^t, \pi(\hat{r}^t,\bar{d}^t) \rangle  + \langle {\bar{r}}^t , \pi(\hat{r}^t, \bar{d}^t)\rangle - \langle \bar{r}^t_i, \pi(\hat{r}^t, \hat{d}^t) \rangle$$
%where $C = O\left(k^2d^2 \log\left(\frac{d \cdot ||A^TA||_F }{\epsilon}\right)\right)$

Fix any round $t \in S_2$. From Lemma \ref{close_obj_corollary} we have:.
\begin{align*}
\langle \hat{r}, \pi(\hat{r}, \bar{d}) \rangle - \langle \hat{r}, \pi(\hat{r}, \hat{d}) \rangle &\le k^3 \epsilon
\end{align*}
Further manipulations give:
\begin{align*}
\big(\langle \hat{r}, \pi(\hat{r}, \bar{d}) \rangle - \langle \bar{r}, \pi(\hat{r}, \bar{d}) \rangle \big) - \big( \langle \hat{r}, \pi(\hat{r}, \hat{d}) \rangle  - \langle \bar{r}, \pi(\hat{r}, \hat{d}) \rangle \big)  &\le k^3 \epsilon -\langle \bar{r} , \pi(\hat{r}, \bar{d}) \rangle + \langle \bar{r}, \pi(\hat{r}, \hat{d}) \rangle \\
\langle 2w, \pi(\hat{r}, \bar{d})\rangle - \langle 2w, \pi(\hat{r}, \hat{d})\rangle &\le k^3 \epsilon - \langle \bar{r}, \pi(\hat{r}, d) \rangle + \langle \bar{r}, \pi(\hat{r}, \hat{d}) \rangle \\
\langle 2w, \pi(\hat{r}, \bar{d})\rangle  &\le \langle 2w, \pi(\hat{r}, \hat{d})\rangle + k^3 \epsilon - \langle \bar{r}, \pi(\hat{r}, \bar{d}) \rangle + \langle \bar{r}, \pi(\hat{r}, \hat{d}) \rangle
\end{align*}

Now, substituting the above expressions back into our expression for regret:
\begin{align*}
%&Regret(\algo_{\text{known-$\theta$}}, T) \\
&\mathbf{Regret}(\algo_{full}, T) \\
&\le C + \sum_{t\in S_2} \langle 2w^t, \pi(\hat{r}^t,\bar{d}^t) \rangle  + \langle {\bar{r}}^t , \pi(\hat{r}^t, \bar{d}^t)\rangle - \langle \bar{r}^t_i, \pi(\hat{r}^t, \hat{d}^t) \rangle \\
&\le C + \sum_{t\in S_2}
\langle 2w^t, \pi(\hat{r}^t, \hat{d}^t)\rangle + k^3 \epsilon - \langle \bar{r}^t, \pi(\hat{r}^t, \bar{d}^t) \rangle + \langle \bar{r}^t, \pi(\hat{r}^t, \hat{d}^t) \rangle  %\\
+  \langle {\bar{r}}^t , \pi(\hat{r}^t, \bar{d}^t)\rangle - \langle \bar{r}^t_i, \pi(\hat{r}^t, \hat{d}^t) \rangle \\
&\le C + \sum_{t\in S_2} \langle 2w^t, \pi(\hat{r}^t, \hat{d}^t)\rangle + k^3 \epsilon  \\
&\le C + 2\sum_{t\in S_2} \E_{i \in \pi(\hat{r}^t, \hat{d}^t)}[w^t_i] + k^3 \epsilon  \\
&\le C+ k^3 \epsilon T+2\Bigg(\sqrt{2d \log\big(1+ \frac{T}{d\lambda} \big)} \Big(\sqrt{2dT \log(\frac{1+T/\lambda}{\delta})} + \sqrt{T\lambda} \Big) + \sqrt{2T \log\frac{1}{\delta}}\Bigg)\\
&= O\left(k^2d^2 \log\left(\frac{d \cdot ||A^\top A||_F }{\epsilon}\right)\right) + k^3\epsilon T + O(d\sqrt{T}\log(\frac{T}{\delta}))
\end{align*}

The last inequality holds with probability $1-\delta$ and uses Lemmas \ref{sum_ci} and \ref{ci_converg}, and sets $\lambda = 1$.

%Setting $\epsilon = \frac{d \cdot ||A^TA||_F}{kT}$ yields $$Regret(\algo) = O(d^2k^2\sqrt{T}\log(\frac{kT}{\delta}))$$
\end{proof}

Finally, the bound on the fairness loss is identical to the bound we proved in Theorem \ref{fairloss} (because our algorithm for constructing distance estimates $\hat{d}$ is unchanged). We have:
\begin{theorem}
\label{thm:fairfull}
For any sequence of contexts and any Mahalanobis distance $d(x_1,x_2) = ||Ax_1-Ax_2||_2$:
$$\mathbf{FairnessLoss}(\algo_{full}, T, \epsilon) \leq   O\left(k^2d^2 \log\left(\frac{d \cdot ||A^\top A||_F }{\epsilon}\right)\right)$$
\end{theorem}

%\begin{theorem}
%Fairness of this algorithm is still $O(k^2d^2log(d/\epsilon))$
%\end{theorem}
%Because the distance are still estimated the same as $\algo$, the fairness loss is still $O(k^2d^2log(d/\epsilon))$.\\

\section{Conclusion and Future Directions}

We have initiated the study of fair sequential decision making
in settings where the notions of payoff and fairness are separate
and may be in tension with each other, and have shown that in a stylized
setting, optimal fair decisions can be efficiently learned \emph{even without direct knowledge of the fairness metric}. A number of
extensions of our framework and results would be interesting to examine. At a high level, the interesting question is: how much can we further relax the information about the fairness metric available to the algorithm? 
For instance, what if the fairness feedback is only partial, identifying
some but not all fairness violations? What if it only indicates whether or not there
were any violations, but does not identify them? What if the feedback is not guaranteed to be exactly consistent with any metric? Or what if the feedback is consistent with \emph{some} distance function, but not one in a known class: for example, what if the 
distance is not exactly Mahalanobis, but is approximately so? In general, it is very interesting to continue to push to close the wide gap between the study of individual fairness notions and the study of group fairness notions. When can we obtain the strong semantics of individual fairness without making correspondingly strong assumptions?

% Acknowledgments---Will not appear in anonymized version
\section*{Acknowledgements}
We thank Steven Wu and Matthew Joseph for helpful discussions at an early stage of this work.
\bibliographystyle{plainnat}
\bibliography{mybib}

\begin{thebibliography}{23}
\providecommand{\natexlab}[1]{#1}
\providecommand{\url}[1]{\texttt{#1}}
\expandafter\ifx\csname urlstyle\endcsname\relax
  \providecommand{\doi}[1]{doi: #1}\else
  \providecommand{\doi}{doi: \begingroup \urlstyle{rm}\Url}\fi

\bibitem[Abbasi{-}Yadkori et~al.(2011)Abbasi{-}Yadkori, P{\'{a}}l, and
  Szepesv{\'{a}}ri]{Abbasi-YadkoriPS11}
Yasin Abbasi{-}Yadkori, D{\'{a}}vid P{\'{a}}l, and Csaba Szepesv{\'{a}}ri.
\newblock Improved algorithms for linear stochastic bandits.
\newblock In \emph{Advances in Neural Information Processing Systems 24: 25th
  Annual Conference on Neural Information Processing Systems 2011. Proceedings
  of a meeting held 12-14 December 2011, Granada, Spain.}, pages 2312--2320,
  2011.
\newblock URL
  \url{http://papers.nips.cc/paper/4417-improved-algorithms-for-linear-stochastic-bandits}.

\bibitem[Berk et~al.(2017)Berk, Heidari, Jabbari, Kearns, and Roth]{berksurvey}
Richard Berk, Hoda Heidari, Shahin Jabbari, Michael Kearns, and Aaron Roth.
\newblock Fairness in criminal justice risk assessments: the state of the art.
\newblock \emph{arXiv preprint arXiv:1703.09207}, 2017.

\bibitem[Chouldechova(2017)]{Chou17}
Alexandra Chouldechova.
\newblock Fair prediction with disparate impact: A study of bias in recidivism
  prediction instruments.
\newblock \emph{arXiv preprint arXiv:1703.00056}, 2017.

\bibitem[Dwork et~al.(2012)Dwork, Hardt, Pitassi, Reingold, and Zemel]{DHPRZ12}
Cynthia Dwork, Moritz Hardt, Toniann Pitassi, Omer Reingold, and Richard Zemel.
\newblock Fairness through awareness.
\newblock In \emph{Proceedings of the 3rd innovations in theoretical computer
  science conference}, pages 214--226. ACM, 2012.

\bibitem[Friedler et~al.(2016)Friedler, Scheidegger, and
  Venkatasubramanian]{FSV16}
Sorelle~A Friedler, Carlos Scheidegger, and Suresh Venkatasubramanian.
\newblock On the (im) possibility of fairness.
\newblock \emph{arXiv preprint arXiv:1609.07236}, 2016.

\bibitem[Hajian and Domingo-Ferrer(2013)]{hajian2013methodology}
Sara Hajian and Josep Domingo-Ferrer.
\newblock A methodology for direct and indirect discrimination prevention in
  data mining.
\newblock \emph{IEEE transactions on knowledge and data engineering},
  25\penalty0 (7):\penalty0 1445--1459, 2013.

\bibitem[Hardt et~al.(2016)Hardt, Price, and Srebro]{HPS}
Moritz Hardt, Eric Price, and Nathan Srebro.
\newblock Equality of opportunity in supervised learning.
\newblock \emph{Advances in Neural Information Processing Systems}, 2016.

\bibitem[H{\'e}bert-Johnson et~al.(2017)H{\'e}bert-Johnson, Kim, Reingold, and
  Rothblum]{HKRR17}
Ursula H{\'e}bert-Johnson, Michael~P Kim, Omer Reingold, and Guy~N Rothblum.
\newblock Calibration for the (computationally-identifiable) masses.
\newblock \emph{arXiv preprint arXiv:1711.08513}, 2017.

\bibitem[Hoeffding(1963)]{hoeffding1963probability}
Wassily Hoeffding.
\newblock Probability inequalities for sums of bounded random variables.
\newblock \emph{Journal of the American statistical association}, 58\penalty0
  (301):\penalty0 13--30, 1963.

\bibitem[Jabbari et~al.(2017)Jabbari, Joseph, Kearns, Morgenstern, and
  Roth]{JJKMR17}
Shahin Jabbari, Matthew Joseph, Michael Kearns, Jamie Morgenstern, and Aaron
  Roth.
\newblock Fairness in reinforcement learning.
\newblock In \emph{International Conference on Machine Learning}, pages
  1617--1626, 2017.

\bibitem[Jain et~al.(2009)Jain, Kulis, Dhillon, and Grauman]{JKDG09}
Prateek Jain, Brian Kulis, Inderjit~S Dhillon, and Kristen Grauman.
\newblock Online metric learning and fast similarity search.
\newblock In \emph{Advances in neural information processing systems}, pages
  761--768, 2009.

\bibitem[Joseph et~al.(2016{\natexlab{a}})Joseph, Kearns, Morgenstern, and
  Roth]{JosephKMR16}
Matthew Joseph, Michael Kearns, Jamie~H Morgenstern, and Aaron Roth.
\newblock Fairness in learning: Classic and contextual bandits.
\newblock pages 325--333, 2016{\natexlab{a}}.

\bibitem[Joseph et~al.(2016{\natexlab{b}})Joseph, Kearns, Morgenstern, Neel,
  and Roth]{infiniterawls}
Matthew Joseph, Michael~J. Kearns, Jamie Morgenstern, Seth Neel, and Aaron
  Roth.
\newblock Fair algorithms for infinite and contextual bandits.
\newblock \emph{CoRR}, abs/1610.09559, 2016{\natexlab{b}}.
\newblock URL \url{http://arxiv.org/abs/1610.09559}.

\bibitem[Kamiran and Calders(2012)]{kamiran2012data}
Faisal Kamiran and Toon Calders.
\newblock Data preprocessing techniques for classification without
  discrimination.
\newblock \emph{Knowledge and Information Systems}, 33\penalty0 (1):\penalty0
  1--33, 2012.

\bibitem[Kearns et~al.(2017)Kearns, Neel, Roth, and Wu]{KNRW17}
Michael Kearns, Seth Neel, Aaron Roth, and Zhiwei~Steven Wu.
\newblock Preventing fairness gerrymandering: Auditing and learning for
  subgroup fairness.
\newblock \emph{arXiv preprint arXiv:1711.05144}, 2017.

\bibitem[Kim et~al.(2018)Kim, Reingold, and Rothblum]{KRR18}
Michael~P Kim, Omer Reingold, and Guy~N Rothblum.
\newblock Fairness through computationally-bounded awareness.
\newblock \emph{arXiv preprint arXiv:1803.03239}, 2018.

\bibitem[Kleinberg et~al.(2017)Kleinberg, Mullainathan, and Raghavan]{KMR16}
Jon Kleinberg, Sendhil Mullainathan, and Manish Raghavan.
\newblock Inherent trade-offs in the fair determination of risk scores.
\newblock In \emph{Proceedings of the 2017 {ACM} Conference on Innovations in
  Theoretical Computer Science, Berkeley, CA, USA, 2017}, 2017.

\bibitem[Kulis et~al.(2013)]{Kul13}
Brian Kulis et~al.
\newblock Metric learning: A survey.
\newblock \emph{Foundations and Trends{\textregistered} in Machine Learning},
  5\penalty0 (4):\penalty0 287--364, 2013.

\bibitem[Liu et~al.(2017)Liu, Radanovic, Dimitrakakis, Mandal, and
  Parkes]{LRDP17}
Yang Liu, Goran Radanovic, Christos Dimitrakakis, Debmalya Mandal, and David~C
  Parkes.
\newblock Calibrated fairness in bandits.
\newblock \emph{arXiv preprint arXiv:1707.01875}, 2017.

\bibitem[Lobel et~al.(2017)Lobel, Leme, and Vladu]{LobelLV17}
Ilan Lobel, Renato~Paes Leme, and Adrian Vladu.
\newblock Multidimensional binary search for contextual decision-making.
\newblock In \emph{Proceedings of the 2017 {ACM} Conference on Economics and
  Computation, {EC} '17, Cambridge, MA, USA, June 26-30, 2017}, page 585, 2017.
\newblock \doi{10.1145/3033274.3085100}.
\newblock URL \url{http://doi.acm.org/10.1145/3033274.3085100}.

\bibitem[Rothblum and Yona(2018)]{RY18}
Guy~N Rothblum and Gal Yona.
\newblock Probably approximately metric-fair learning.
\newblock \emph{arXiv preprint arXiv:1803.03242}, 2018.

\bibitem[Zafar et~al.(2017)Zafar, Valera, Gomez~Rodriguez, and
  Gummadi]{zafar2017fairness}
Muhammad~Bilal Zafar, Isabel Valera, Manuel Gomez~Rodriguez, and Krishna~P
  Gummadi.
\newblock Fairness beyond disparate treatment \& disparate impact: Learning
  classification without disparate mistreatment.
\newblock In \emph{Proceedings of the 26th International Conference on World
  Wide Web}, pages 1171--1180. International World Wide Web Conferences
  Steering Committee, 2017.

\bibitem[Zemel et~al.(2013)Zemel, Wu, Swersky, Pitassi, and Dwork]{ZWSPD13}
Rich Zemel, Yu~Wu, Kevin Swersky, Toni Pitassi, and Cynthia Dwork.
\newblock Learning fair representations.
\newblock In \emph{International Conference on Machine Learning}, pages
  325--333, 2013.

\end{thebibliography}

\appendix

\section{Generalization to Multiple Actions}\label{multiple_arm}
In the body of the paper, we analyzed the standard contextual bandit setting in which the algorithm must choose a \emph{single} action to take at each round. However, it is often the case that this constraint is artificial and undesirable in settings for which fairness is a concern. Consider, for example, the case of lending: at each round, a bank observes the loan applications of a collection of individuals, and decides whom to grant loans to. Some loans may be profitable and some loans may not be --- so the optimal policy is non-trivial. But there need not be a budget constraint --- the optimal policy may grant loans to as many qualified individuals as there are on a given day. In our framework, this corresponds to letting the algorithm take as many as $k$ actions on a single day. Fortunately, all of our results generalize to this case. The maximum reward per day in this case increases from 1 to $k$, so naturally the regret bound we obtain is also a factor of $k$ larger. In this section, we explain the details of our proof that need to be modified. 

The first step is to consider a modified linear program $LP(a,c)$, which we will write as $LP_m(a,c)$. It simply replaces the simplex constraint that the probabilities of actions sum to 1 with the hypercube constraint that no probability can be greater than $1$:

\begin{equation*}
\begin{aligned}
& \underset{\pi=\{p_1,\ldots,p_k\} }{\text{maximize}}
& & \sum_{i=1}^k p_i a_i \\
& \text{subject to}
& & \vert p_i - p_j \vert \le c_{i,j}, \forall (i,j)\\
&  && 0 \le p_i \le 1, \forall i
\end{aligned}
\end{equation*}

We must also change our definition of regret, because the benchmark we want to compete with is the best fair policy that can make up to $k$ action selections per round. This simply corresponds to comparing to a benchmark which is defined with respect to $LP_m(a,c)$ --- but the form of the regret is unchanged:

\begin{align*}
&\mathbf{Regret}_m(\algo,T) \\
&= \sum_{t=1}^T \sum_{i=1}^k  \bar{r}^t_i \cdot Pr(\text{best fair policy pulls arm $i$ in round $t$}) - \bar{r}^t_i \cdot Pr(\text{$\algo$ pulls arm $i$ in round $t$}) \\
&= \sum_{t=1}^T \langle \bar{r}^t, \pi(\bar{r}^t,\bar{d}^t) \rangle - \langle \bar{r}^t, f^t(h^t, x^t) \rangle
\end{align*}

where $\pi$ is defined exactly as before, except with respect to $LP_m(a,c)$. 

The first observation is that our generalization to multiple arms does not affect our analysis of fairness loss at all, since we are able to bound this without reference to the rewards. That is, we still have that fairness loss is bounded as $$\mathbf{FairnessLoss}(\algo_{full_m}, T, \epsilon) \leq   O\left(k^2d^2 \log\left(\frac{d \cdot ||A^\top A||_F }{\epsilon}\right)\right)$$ 

 As for our regret analysis, certain terms in the regret scale by a factor of $k$.
\begin{restatable}{thm}{multifullregretthm}
 $$ \mathbf{Regret}_m(\algo_{full_m},T) \leq O\left(k^3d^2 \log\left(\frac{d \cdot ||A^\top A||_F }{\epsilon}\right) + k^3\epsilon T + dk\sqrt{k^2T}\log(\frac{kT}{\delta})\right)$$
\end{restatable}
\begin{proof}
There are only two parts of our proof that depend on the structure on the linear program $LP(a,c)$. The first is the proof of Lemma \ref{close_obj}, which uses the fact that if we take a feasible solution to $LP(a,c)$ and reduce its values pointwise, we maintain feasibility --- that is, that the feasible region of $LP(a,c)$ is downward closed. But note that the feasible region of $LP_m(a,c)$ is also downward closed, so the same argument goes through. Recall that our analysis in the known objective case partitions rounds into two sorts: rounds for which we can bound our per-round regret (from Lemma \ref{close_obj}), and a bounded number of rounds in which we cannot.  For those rounds in which we cannot bound the per-round regret, the maximum regret is now $k$ rather than $1$. So, our regret during these rounds increases by a factor of $k$ to $O\left(k^3d^2 \log\left(\frac{d \cdot ||A^\top A||_F }{\epsilon}\right)\right)$.

Therefore, we have that $$\mathbf{Regret}_m(\algo_{full_m},T)  \le O\left(k^3d^2 \log\left(\frac{d \cdot ||A^\top A||_F }{\epsilon}\right)\right) + k^3 \epsilon T + \sum_{t\in S_2} \langle 2w^t, \pi(\hat{r}^t, \hat{d}^t)\rangle $$

where $S_2 = \{t : \forall (i,j) : |\hat{d}^t_{ij} - \bar{d}^t_{ij}| \leq \epsilon \ \mathrm{or} \ v^t_{ij} = 0\}$

Next, we need to consider the final term in this expression. $\langle w^t, \pi(\hat{r}^t, \hat{d}^t)\rangle$ is the expected sum of the confidence interval widths of the arms that are pulled at round $t$. By the same martingale argument as in lemma \ref{ci_converg}, with high probability, the expected sum of the confidence interval widths over time horizon $T$ is close to the  realized sum of the confidence widths of the arms pulled; in this case, the martingale is $$D^t = \sum_{s=1}^t \sum_{i=1}^k  w^s_i \cdot \Pr(\text{arm $i$ is pulled in round $s$}) - \sum_{s=1}^t \sum_{i=1}^k  w^s_i \cdot \mathbbm{1}(\text{arm $i$ is pulled in round $s$})$$

However, in this case, the martingale difference is bounded by at most $k$ instead of 1. Hence, applying the Azuma-Hoeffding inequality gives us that with probability $1-\delta$,
$$ \sum_{t=1}^T \sum_{i=1}^k  w^t_i \cdot \Pr(\text{arm $i$ is pulled in round $t$}) \le \sum_{t=1}^T \sum_{i=1}^k  w^t_i \cdot \mathbbm{1}(\text{arm $i$ is pulled in round $t$}) + \sqrt{2k^2T \log\frac{1}{\delta}}$$

First, note that the confidence interval derived from lemma \ref{ci_lemma} remains valid. Also, $\bar{V}^t = \bar{V}^{t-1} + \sum_{i \in P^t} x^t_i {x^t_i}^\top$. For simplicity in notation, we write $P^t=\{i : \text{arm $i$ is pulled in round $t$}\}$. So we need to bound $\sum_{t=1}^T \sum_{i \in P^t}  w^t_i$.

We can then derive:
\begin{align*}
\sum_{t=1}^T \sum_{i \in P^t}  w^t_i &\le \sum_{t=1}^T \sum_{i \in P^t}  \| x_{i}^t\|_{(\bar{V}^{t-1})^{-1}} \Big( \sqrt{2d \log(\frac{1+t/\lambda}{\delta})} + \sqrt{\lambda}\Big)\\
&\le \sum_{t=1}^T \sum_{i \in P^t}  \| x_{i}^t\|_{(\bar{V}^{t-1})^{-1}} \Big( \sqrt{2d \log(\frac{1+t/\lambda}{\delta})} \Big) + \sum_{t=1}^T \sum_{i \in P^t} \Big(\| x_{i}^t\|_{(\bar{V}^{t-1})^{-1}} \sqrt{\lambda}\Big)\\
&\le \sum_{t=1}^T \sum_{i \in P^t}  \| x_{i}^t\|_{(\bar{V}^{t-1})^{-1}} \cdot \Big( \sqrt{ \sum_{t=1}^T \sum_{i \in P^t}2d \log(\frac{1+t/\lambda}{\delta})} \Big) +  \sqrt{ \sum_{t=1}^T \sum_{i \in P^t}  \lambda}\\
&\le \sum_{t=1}^T \sum_{i \in P^t}  \| x_{i}^t\|_{(\bar{V}^{t-1})^{-1}} \cdot \Big( \sqrt{ 2dkT \log(\frac{1+kT/\lambda}{\delta})} \Big) +  \sqrt{ kT  \lambda}
\end{align*}

For each $i \in [k]$, write $A_i$ to denote the set of rounds that arm $i$ is pulled. $\sum_{t=1}^T \sum_{i \in P^t}  \| x_{i}^t\|_{(\bar{V^t})^{-1}} = \sum_{i=1}^k \sum_{t \in A_i} \| x_{i}^t\|_{(\bar{V^t})^{-1}}$, so for each $i \in [k]$, we'll bound $\sum_{t \in A_i} \| x_{i}^t\|_{(\bar{V^t})^{-1}}$.

\begin{lemma}
\label{ind_arm_ci_converge}
$$\sum_{t \in A_i} \| x_{i}^t\|_{(\bar{V^{t-1}})^{-1}} \le \sqrt{2d \log\big(1+ \frac{kT}{d\lambda} \big)}$$
\end{lemma}
\begin{proof}
 We'll iterate each $\| x_{i}^t\|_{(\bar{V^{t-1}})^{-1}}$ first over round $t=1, \hdots, T$ and then $j \in P^t$ where the order of $P^t$ has its very first element as $\| x_{i}^t\|$ and the rest is arbitrary. Let's call this indexing $a$. First, we have that $\bar{V}(a) = \bar{V}(a-1) + x(a) {x(a)}^\top$. More importantly, because of the way we chose to index, for each $t \in A_i$ and index $a$ that corresponds to $(i,t)$, $\| x_{i}^t\|_{(\bar{V}^{t-1})^{-1}} = \| x(a)\|_{(\bar{V}(a-1))^{-1}}$

From Lemma 11 in \cite{Abbasi-YadkoriPS11} we have $\sum_{a=1}^N \| x(a)\|_{(\bar{V}(a-1))^{-1}} \le \sqrt{2d \log\big(1+ \frac{N}{d\lambda} \big)}$, where $N \le kT$.

Therefore, we have that $$ \sum_{t \in A_i} \| x_{i}^t\|_{(\bar{V^{t-1}})^{-1}} \le \sum_{a=1}^N \| x(a)\|_{(\bar{V}(a-1))^{-1}} \le \sqrt{2d \log\big(1+ \frac{kT}{d\lambda} \big)}$$
\end{proof}

Applying the above lemma for each arm $i \in [k]$, we have

\begin{align*}
\sum_{t=1}^T \sum_{i \in P^t}  w^t_i &\le \sum_{t=1}^T \sum_{i \in P^t}  \| x_{i}^t\|_{(\bar{V^{t-1}})^{-1}} \cdot \Big( \sqrt{ 2dkT \log(\frac{1+kT/\lambda}{\delta})} \Big) +  \sqrt{ kT  \lambda} \\
&\le k\sqrt{2d \log\big(1+ \frac{kT}{d\lambda} \big)} \cdot \Big( \sqrt{ 2dkT \log(\frac{1+kT/\lambda}{\delta})} \Big) +  \sqrt{ kT  \lambda}
\end{align*}

\end{proof}

%\section{My Proof of Theorem 2}

%This is a complete version of a proof sketched in the main text.

\end{document}